\documentclass[11pt]{article}

\usepackage{dsfont}

\def\colorful{1}

\usepackage[utf8]{inputenc} 
\usepackage[T1]{fontenc}    
\usepackage{hyperref}       
\usepackage{url}            
\usepackage{booktabs}       
\usepackage{amsfonts}       
\usepackage{nicefrac}       
\usepackage{microtype}      
\usepackage{xcolor}         

\usepackage{amsmath,amsthm,amssymb}
\usepackage{bbm}
\usepackage{dsfont}
\usepackage{graphicx}
\usepackage{subfigure}
\usepackage{algorithm}
\usepackage[noend]{algpseudocode}

\newif\ifhyper\IfFileExists{hyperref.sty}{\hypertrue}{\hyperfalse}
\hypertrue
\ifhyper\usepackage{hyperref}\fi

\usepackage{enumitem}
\usepackage{framed}
\usepackage{nicefrac}

\def\nnewcolor{1}
\ifnum\nnewcolor=1

\fi
\ifnum\nnewcolor=0

\fi

\ifnum\colorful=1

\else

\fi

\makeatletter
\newtheorem*{rep@theorem}{\rep@title}
\newcommand{\newreptheorem}[2]{%
\newenvironment{rep#1}[1]{%
 \def\rep@title{#2 \ref{##1}}%
 \begin{rep@theorem}}%
 {\end{rep@theorem}}}
\makeatother

\newtheorem{theorem}{Theorem}[section]
\newreptheorem{theorem}{Theorem}
\newtheorem{lemma}[theorem]{Lemma}
\newreptheorem{lemma}{Lemma}

\newtheorem{definition}[theorem]{Definition}

\newcommand{\E}{\mathbb{E}}

\newcommand{\Z}{\mathbb{Z}}
\newcommand{\R}{\mathbb{R}}

\newcommand{\D}{\mathcal{D}}

\newcommand{\poly}{\mathrm{poly}}
\renewcommand{\vec}[1]{\mathbf{#1}}
\newcommand{\x}{\vec{x}}
\newcommand{\relu}{\mathrm{ReLU}}
\newcommand{\sign}{\mathrm{sgn}}
\newcommand{\h}{\hspace{-10pt}}
\DeclareMathOperator{\proj}{proj}

\newcommand{\littlesum}{\mathop{\textstyle \sum}}

\title{ReLU Regression with Massart Noise}

\author{
Ilias Diakonikolas\thanks{Supported by NSF Medium Award CCF-2107079,
NSF Award CCF-1652862 (CAREER), a Sloan Research Fellowship, and
a DARPA Learning with Less Labels (LwLL) grant.}\\
University of Wisconsin-Madison\\
{\tt ilias@cs.wisc.edu}\\
\and
Jongho Park\\
University of Wisconsin-Madison\\
{\tt jongho.park@wisc.edu}\\
\and
Christos Tzamos\\
University of Wisconsin-Madison\\
{\tt tzamos@cs.wisc.edu}
}

\begin{document}

\maketitle

\begin{abstract}
We study the fundamental problem of ReLU regression, where the goal is
to fit Rectified Linear Units (ReLUs) to data. 
This supervised learning task is efficiently solvable in the realizable setting,
but is known to be computationally hard with adversarial label noise.
In this work, we focus on ReLU regression in the Massart noise model, 
a natural and well-studied semi-random noise model. In this model, the label of every point 
is generated according to a function in the class, but an adversary is allowed to change 
this value arbitrarily with some probability, which is {\em at most} $\eta < 1/2$. 
We develop an efficient algorithm that achieves exact parameter recovery in this model
under mild anti-concentration assumptions on the underlying distribution. 
Such assumptions are necessary for exact recovery to be information-theoretically possible. 
We demonstrate that our algorithm significantly outperforms naive applications of
$\ell_1$ and $\ell_2$ regression on both synthetic and real data.
\end{abstract}

\setcounter{page}{0}

\thispagestyle{empty}

\newpage

\section{Introduction} \label{sec:intro}

Learning in the presence of outliers is a key challenge in
machine learning with several data analysis applications, 
including in ML security~\cite{Barreno2010,BiggioNL12, SteinhardtKL17, DiakonikolasKKLSS2018sever} 
and in exploratory data analysis of real datasets with natural outliers, 
e.g., in biology~\cite{RP-Gen02, Pas-MG10, Li-Science08}.
The goal in such settings is to design computationally efficient learners 
that can tolerate a constant fraction of outliers, independent of the dimensionality of the data.
Early work in robust statistics~\cite{HampelEtalBook86, Huber09}
gave sample-efficient robust estimators for various basic tasks, alas
with runtimes exponential in the dimension. A recent line of work in computer science, 
starting with~\cite{DKKLMS16, LaiRV16}, developed the first computationally efficient 
robust learning algorithms for various high-dimensional tasks. Since these early works, 
there has been significant progress in algorithmic robust high-dimensional statistics 
by several communities, see~\cite{DK19-survey} for a  recent survey.

In this work, we study the problem of learning Rectified Linear Units (ReLUs)
in the presence of label noise. The ReLU function $\mathrm{ReLU}_{\vec{w}}: \R^d \rightarrow \R$,
parameterized by a vector $\vec{w} \in \R^d$, is defined as 
$\mathrm{ReLU}_{\vec{w}}(\vec{x}) := \mathrm{ReLU}(\vec{w} \cdot \vec{x} ) = 
\max \left \{ 0, \vec{w} \cdot \vec{x}  \right \}$.  
ReLU regression --- the task of fitting ReLUs to a set of labeled examples --- 
is a fundamental task and an important primitive in the theory
of deep learning. In recent years, ReLU regression has been extensively studied in theoretical
machine learning both from the perspective of designing efficient algorithms and 
from the perspective of computational hardness,
see, e.g.,~\cite{GKKT17, Mahdi17, MR18, YS19, GoelKK19, FCG20, yehudai2020learning, DGKKS20, DKZ20-sq-reg, GGK20, DKP21-SQ}.
The computational difficulty of this statistical problem crucially 
depends on the underlying assumptions about the input data.
In the realizable case, i.e., when the labels are consistent with the target function, the problem is 
efficiently solvable with practical algorithms, see, e.g.,~\cite{Mahdi17}. 
On the other hand, in the presence of even a small
constant fraction of adversarially labeled data,  
computational hardness results are known even for approximate recovery~\cite{HardtM13, MR18} 
and under well-behaved distributions~\cite{GoelKK19, DKZ20-sq-reg, GGK20, DKP21-SQ}.
See Section~\ref{ssec:related} for a detailed summary of related work.

A challenging corruption model is the adversarial label noise (aka ``agnostic'') model, 
in which an adversary is allowed to corrupt an {\em arbitrary} $\eta<1/2$ fraction of the labels.
The aforementioned hardness results rule out the existence of efficient algorithms for learning
ReLUs with optimal error guarantees in this model, even when the underlying distribution
on examples is Gaussian. Moreover, when no assumptions are made on the underlying
data distribution, no fully polynomial time algorithm with non-trivial guarantee is possible.
In fact, for the distribution-independent setting, 
even for the simpler case of learning linear functions, there are strong 
computational hardness results for any constant $\eta$~\cite{khachiyan1995complexity, HardtM13}.
These negative results motivate the following natural question:
\begin{center}
{\em Are there realistic  label-noise models in which efficient learning is possible \\ 
without strong distributional assumptions?}
\end{center}

Here we focus on ReLU regression in the presence of {\em Massart (or bounded) noise}~\cite{Massart2006}, 
and provide an efficient learning algorithm with minimal distributional assumptions.
In the process, we also provide an efficient noise-tolerant algorithm for the simpler case of linear regression.

The Massart model~\cite{Massart2006} is a classical semi-random noise model originally defined 
in the context of binary classification.
In this model, an adversary has control over a {\em random} $\eta<1/2$ fraction of the labels 
(see Definition~\ref{def:Massart} for the formal definition). 
Recent work~\cite{DGT19} gave the first efficient learning algorithm for linear separators 
with non-trivial error guarantees in the Massart model {\em without distributional assumptions}.
In this work, we ask to what extent such algorithmic results
are possible for learning real-valued functions. To state
our contributions, we formally define the following natural generalization of the model
for real-valued functions.

\begin{definition}[Learning Real-valued Functions with Massart Noise] \label{def:Massart}
Let $\mathcal{F}$ be a concept class of real-valued functions over $\R^d$ 
and $f:\R^d \rightarrow \R$ be an unknown function in $\mathcal{F}$.
For a given parameter $\eta < 1/2$, the algorithm specifies $m \in \Z_+$ and obtains $m$ samples
$(\x_i, y_i)_{i=1}^m$, such that:
\begin{itemize}
\item[(a)] every $\x_i$ is drawn i.i.d. from a fixed distribution $\D_{\x}$, and
\item[(b)] each $y_i$ is equal to  $f(\x_i)$ with
probability $1 - \eta$ and takes an \emph{arbitrary} value with probability $\eta$, 
chosen by an adversary after observing the samples drawn and the values that can be corrupted.
\end{itemize}
\end{definition}

We note that in Definition~\ref{def:Massart} the adversary can corrupt each sample independently 
with probability $\eta$, but may also choose not to do so for some of the samples.
In the context of binary classification, the above model 
has been extensively studied
in the theoretical ML community for the class of 
linear separators~\cite{AwasthiBHU15, AwasthiBHZ16, ZhangLC17, Zhang20, DKTZ20, DGT19, ChenKMY20, DKKTZ21-benign} and in the context of boosting~\cite{DiakonikolasIKL21}. 
Even though the Massart noise model might appear innocuous at first sight, 
the ability of the Massart adversary to choose {\em whether} to perturb a given label and, if so, 
with what probability (which is unknown to the learner), makes the design of efficient algorithms in this model challenging. 
Specifically, for distribution-independent PAC learning of linear separators, even approximate learning in this model
is computationally hard~\cite{DK20-hard}.

Extending this model to real-valued functions, we study regression tasks
under Massart noise. We focus on the realizable setting where the uncorrupted 
data exhibit clean functional dependencies, i.e., $f(\x) = \vec w^* \cdot \x$ for linear regression 
and $f(\x) = \relu(\vec w^* \cdot \x)$ for ReLU regression.
The realizable setting is both of theoretical and practical interest. 
Prior work~\cite{Mahdi17,du18convolutional,kalan2019fitting,yehudai2020learning}
developed algorithms for learning ReLUs in this setting (without Massart noise),
providing theoretical insights on the success of deep learning architectures.
On the practical side, there are many applications in which we observe
clean functional dependencies on the uncorrupted data.
For instance, clean measurements are prevalent in many signal processing applications,
including medical imaging, and are at the heart of the widely popular field of
compressive sensing~\cite{candes2008introduction}. 

\subsection{Main Results}
Our main result is an efficient algorithm for learning ReLUs with Massart noise
under information-theoretically minimal distributional assumptions.
To build up to the more challenging case of ReLUs, 
we start with the simpler case of linear functions.
Linear regression is in and of itself
one of the most well-studied statistical tasks, 
with numerous applications in machine learning~\cite{Rousseeuw:1987}, as well as 
in other disciplines, including economics~\cite{Dielman01} and biology~\cite{McD09}.

In our Massart noise setting, the goal is to identify a linear relation $y = \vec w^* \cdot \vec x$ 
that the clean samples $(\vec x, y)$ (inliers) satisfy. 
We show that, under the minimal (necessary) assumption that the distribution
is not fully concentrated on any subspace, the problem is efficiently identifiable.

\begin{theorem} [Exact Recovery for Massart Linear Regression] \label{thm:linear_exact}
Let $\D_{\x}$ be a distribution on $\R^d$ such that $\Pr_{\x \sim \D_{\x}}[\vec{r} \cdot \x = 0] \le 1 - \rho$ 
for all non-zero $\vec{r} \in \R^d$. Let $\eta < 1/2$ be the upper bound on the Massart noise rate. 
Denote $\vec{w}^*$ the vector representing the true linear function. There is an algorithm that draws 
$\tilde{O}(\frac{d^3}{\rho(1-2\eta)^2})$ samples, runs in $\poly(d,b,\rho^{-1},(1-2\eta)^{-1})$ time, 
where $b$ is an upper bound on the bit complexity of the samples and parameters, 
and outputs $\vec{w}^*$ with probability at least $9/10$.
\end{theorem}

We establish Theorem~\ref{thm:linear_exact} in two steps.
We start by providing a simple algorithmic approach for the special case that $\rho= 1$, i.e., 
the examples are in general position (Theorem~\ref{thm:linear_special}).
We then relax the density assumption on $\D_\x$ 
so that the only assumption needed is that the support of $\D_\x$ 
spans $\R^d$, thereby proving Theorem~\ref{thm:linear_exact}.

We note that the anti-concentration assumption about the distribution $\D_{\x}$ in Theorem~\ref{thm:linear_exact} 
is necessary so that exact recovery is {\em information-theoretically} possible.
Indeed, if the distribution was concentrated entirely on a linear subspace, 
it would be (information-theoretically) impossible to identify the orthogonal component of $\vec w^*$ 
on that subspace.  

When this anti-concentration assumption is violated and the problem is non-identifiable, 
we provide a (weaker) PAC learning guarantee for the linear case in Theorem~\ref{thm:linear_pac} 
of Appendix~\ref{supp:linear_pac}.

\medskip

Our main algorithmic result is for the problem of ReLU regression, 
where the inliers satisfy $y = \text{ReLU}(\vec w^* \cdot \vec x)$ and
an $\eta < 1/2$ fraction of the labels are corrupted by Massart noise.
Even in this more challenging case, we show it is possible to efficiently identify 
the true parameters $\vec w^*$, as long as every homogeneous halfspace contains a non-negligible fraction 
of the sample points.

\begin{theorem}[Exact Recovery for Massart ReLU Regression] \label{thm:relu_exact}
Let $\D_{\x}$ be a distribution on $\R^d$ such that 
$\Pr_{\x \sim \D_{\x}} \left[ \vec{r} \cdot \x = 0 \mid \vec{w} \cdot \x \ge  0 \right] \le 1 - \rho$ and 
$\Pr_{\x \sim \D_{\x}}[\vec{w} \cdot \x \ge  0] \ge \lambda$ for all non-zero $\vec{r}, \vec{w} \in \R^d$. 
Let $\eta < 1/2$ be the upper bound on the Massart noise rate. Denote $\vec{w}^*$ the parameter vector of the target ReLU. 
There is an algorithm that draws $\tilde{O}(\frac{d^3}{\rho \lambda^2 (1-2\eta)^2})$ samples, runs in
$\poly(d, b, \rho^{-1}, \lambda^{-1}, (1-2\eta)^{-1})$ time, and outputs $\vec{w}^*$ with probability at least $9/10$.
\end{theorem}

We note that both assumptions on the probability mass of homogeneous halfspaces under $\D_{\x}$ 
in Theorem~\ref{thm:relu_exact} are necessary for identifiability. Indeed, 
similar to the linear regression case, if the distribution was concentrated entirely on a linear subspace, 
it would be (information-theoretically) impossible to identify the orthogonal component of $\vec w^*$ 
on that subspace.  Moreover, if there was a halfspace, parameterized by $\vec{w} \in \R^d$, 
such that $\Pr_{\x \sim \D_{\x}}[\vec{w} \cdot \x \ge  0] = 0$, 
it would be impossible to distinguish between the case where $\vec w^* = \vec w$ 
and  the case where $\vec w^* = 2 \vec w$ (even without noise), as all points would have $0$ labels.

In the case where the problem is non-identifiable (when $\D_{\x}$ does not satisfy the aforementioned assumptions), obtaining a weaker PAC learning guarantee is in principle possible.
It remains an interesting open problem whether an efficient PAC learning algorithm exists in this case. 
We suspect that the PAC learning problem is computationally hard in full generality.

\subsection{Technical Overview} \label{ssec:techniques}

Here we provide a detailed intuitive description of our technical approach.
Recall that we study the problem of robust regression in the presence of label corruptions. 

\paragraph{Agnostic Model versus Massart Noise}
We start by contrasting our setting with the problem of recovery 
in the presence of agnostic label corruptions.
Specifically, suppose that an arbitrary $\eta < 1/2$ fraction of the labels is adversarially 
corrupted and that the remaining $(1-\eta)$-fraction perfectly fit the target function (realizable case). 
The goal of the learner is to compute the function that fits as many points (inliers) as possible.
Given a sufficient number of samples that span all $d$ dimensions from the data distribution, 
this function is unique for the class of ReLUs and matches the true function with high probability. 
However, even in the simpler case of linear functions, the corresponding computational problem 
of $\ell_0$-minimization is computationally hard without distributional assumptions, 
as it is an instance of robust subspace recovery~\cite{HardtM13}.

The key conceptual contribution of this work is that strong algorithmic results are possible 
with minimal distributional assumptions by relaxing the assumption that an {\em arbitrary} $\eta$ fraction 
of the points are corrupted. Indeed, the Massart noise model (Definition~\ref{def:Massart}) 
is essentially equivalent to a more restricted adversary that is presented 
with a {\em uniformly random} $\eta$-fraction of the points, which they can corrupt arbitrarily at will.

\paragraph{$\ell_0$ to $\ell_1$ minimization}
Given this milder corruption model, we propose novel algorithms for efficient exact recovery of the underlying function. 
We obtain our algorithms by replacing the $\ell_0$-minimization with $\ell_1$-minimization, 
which can be shown to converge to the true function {\em in the limit} 
and is efficient to optimize in the linear regression case. For intuition, consider a single-point distribution 
that always outputs labeled examples of the form 
$(\vec x,y)$, where the example $\vec x$ is always the same but the labels $y$ may differ. 
The Massart assumption indicates that the value of $y$ is correct more than half of the time, 
so the estimate that maximizes the number of correct samples ($\ell_0$-minimizer) 
recovers the underlying function. However, if one considers the $\ell_1$-minimizer, 
i.e., the value $v$ that minimizes $\E[|y-v|]$, this corresponds to the median value of $y$, 
which is also correct if more than half of the samples are correctly labeled. 

Generalizing this intuition, we propose a natural and tight condition under which empirical
$\ell_1$-minimization results in the true $\ell_0$-minimizer (Lemma~\ref{lemma:structural_condition}).  
While this condition holds under Massart noise for arbitrary distributions in the population level, 
it can fail to hold with high probability when considering only a finite set of samples from the distribution.
For example, consider the one-dimensional case of $\vec w^* = 1$, where most $\x_i$'s 
are near-zero and uncorrupted, while a few corrupted samples lie extremely far from zero.
In this case, the empirical $\ell_1$-minimizer will be dominated by the few corrupted samples 
and would differ from the $\ell_0$-minimizer.
In particular, the sample complexity of naive $\ell_1$-minimization
would crucially depend on the concentration properties of the distribution on $\vec x$.

\paragraph{Transforming the Points via Radial-isotropy} 
The main technical idea behind obtaining sample and computationally efficient algorithms 
is to transform the original dataset into an equivalent one that satisfies the required properties 
with high probability, as it becomes sufficiently concentrated. In particular, performing 
a linear transformation mapping every point $\vec x$ to $\vec A \vec x$, 
while keeping the corresponding label $y$, is without loss of generality, 
as we are interested in identifying the true (generalized) linear function 
that depends only on the inner product of every point with a parameter vector $\vec w$. 
Finding such a vector $\vec w'$ in the transformed space $\vec A \vec x$ 
results in the equivalent vector $\vec w = \vec A \vec w'$ in the original space. 
Moreover, an additional operation we can perform is to take a single sample $(\vec x,y)$ 
and multiply it by a positive scalar $\lambda > 0$ to replace it with the sample $(\lambda \vec x, \lambda y)$. 
For both the linear and ReLU cases, any sample that is an inlier for the true function 
remains an inlier after this transformation.

We can use these two operations to bring our pointset in radial-isotropic position, 
i.e., so that all the $\vec x$'s in the dataset are unit-norm and 
the variance in any direction is nearly identical. 
Formally, we require the following definition.

\begin{definition}[Radial Isotropy] \label{lemma:forster}
Given $\{\x_1, \dots, \x_n\} \subset \mathcal{S}^{d-1}$, $\vec{A}: \R^d \rightarrow \R^d$ is a radial-isotropic transformation if $\littlesum_{i=1}^n \frac{(\vec{A} \x_i)(\vec{A} \x_i)^T}{\|\vec{A}\x_i\|_2^2} = \frac{n}{d}I$. 
For $0< \gamma <1$, we say that the points are in $\gamma$-approximate radial-isotropic position, 
if for all  $\vec{v} \in \mathcal{S}^{d-1}$ it holds that 
$(d/n)\littlesum_{i=1}^n   (\x_i \cdot \vec{v} )^2 \ge 1-\gamma$.
\end{definition}

In such a normalized position, we can argue that with high probability 
the weight of all inliers in every direction is more than the weight of the outliers, 
which guarantees that the empirical $\ell_1$-minimizer will efficiently converge to the true function.

\vspace{-0.1cm}

\paragraph{Learning ReLUs}
Unfortunately, while $\ell_1$-minimization for linear functions is convex and efficiently solvable 
via linear programming, $\ell_1$-minimization for ReLUs is challenging due to its non-convexity;
that is, we cannot easily reduce ReLU regression to a simple optimization method.
We instead establish a structural condition (see Lemma~\ref{lemma:separation_condition}) 
under which we can compute an efficient separation oracle
between the optimal parameter vector $\vec{w}^*$ and a query $\vec{w}$.
More specifically, we show that any suboptimal guess for the parameter vector $\vec w$ 
can be improved by moving along the opposite direction of the gradient of the $\ell_1$-loss 
for the subset of points in which the condition in Lemma~\ref{lemma:separation_condition} is satisfied.
Identifying such a direction of improvement yields a separating hyperplane,
so we exploit this to efficiently identify 
$\vec{w}^*$ by running the ellipsoid method with our separation oracle.

Importantly, for this result to hold with a small number of samples, 
we need to again bring to radial-isotropic position the points that fall in the linear (positive) part
of the ReLU for the current guess vector $\vec w$. In contrast to the linear case, though, 
where this transformation was applied once, in this case it needs to be applied again with every new guess. 
This results in a function that changes at every step, which is not suitable for direct optimization.

Using these ideas, our algorithms can efficiently recover the underlying function exactly using few samples. 
Our algorithms make mild genericity assumptions about the position of the points, requiring that the points 
are not concentrated on a lower-dimensional subspace or, for the case of ReLUs, 
do not lie entirely in an origin-centered halfspace. 
As already mentioned, such assumptions are necessary for the purposes of identifiability.

\subsection{Related Work} \label{ssec:related}

Given the extensive literature on robust regression, here we discuss 
the most relevant prior work.

\paragraph*{ReLU Regression}
In the realizable setting,~\cite{Mahdi17} and, more recently,~\cite{kalan2019fitting} showed
that gradient descent efficiently performs exact recovery for ReLU regression under the Gaussian
distribution on examples. \cite{yehudai2020learning} generalized this result to a broader family
of well-behaved distributions. In the agnostic or adversarial label noise model, a line of work has shown that learning with near-optimal error guarantees requires super-polynomial time, 
even under the Gaussian distribution~\cite{GoelKK19, DKZ20-sq-reg, GGK20, DKP21-SQ}.
On the positive side,~\cite{DGKKS20} gave an efficient learner with approximation guarantees 
under log-concave distributions. Without distributional assumptions, even approximate learning 
is hard~\cite{HardtM13, MR18}.

The recent work~\cite{karmakar2020study} studies ReLU regression in the realizable setting 
under a noise model similar to -- but more restrictive than -- the Massart model of Definition~\ref{def:Massart}.
Specifically, in the setting of~\cite{karmakar2020study}, the adversary can corrupt
a label with probability at most $\eta$, but only via additive noise bounded above by a constant.
\cite{karmakar2020study} gives an SGD-type algorithm for ReLU regression in this model.
We note that their algorithm does not achieve exact recovery and its guarantees
crucially depend on the concentration properties of the marginal distribution
and the bound on the additive noise.

\paragraph*{Comparison of Noise Models}
It is worth comparing the Massart noise model (Definition~\ref{def:Massart})
with other noise models studied in the literature. The strongest corruption model we are aware
of is the strong contamination model~\cite{DKKLMS16}, in which an omniscient adversary
can corrupt an arbitrary $\eta<1/2$ fraction of the labeled examples. In the adversarial label noise
model, the adversary can corrupt an arbitrary $\eta<1/2$ fraction of the labels (but not the examples).
Efficient robust learning algorithms in these models typically only give approximate error guarantees and 
require strong distributional assumptions. Specifically, for the case of linear regression,~\cite{KlivansKM18, DKS19-lr, DiakonikolasKKLSS2018sever} give robust approximate learners in the strong contamination model under 
the Gaussian distribution and, more broadly, distributions with bounded moments. In the adversarial label
noise model,~\cite{BhatiaJK15} gave efficient robust learners under strong concentration bounds 
on the underlying distribution that can tolerate $\eta<1/50$ fraction of outliers.

The recent work of \cite{chen2020online} considers a Massart-like noise model 
in the context of linear regression with random observation noise. 
\cite{chen2020online} provides an SDP-based approximate recovery algorithm when the noise rate 
satisfies $\eta < 1/3$ and a sum-of-squares-based algorithm when $\eta < 1/2$.
It should be noted their algorithm does not efficiently achieve exact recovery. 
We provide a more detailed description of that work in Appendix~\ref{supp:chen}.

A related noise model is that of {\em oblivious} label noise, where 
the adversary can corrupt an $\eta$ fraction of the labels with additive noise 
that is {\em independent} of the covariate $\x$. 
More precisely, the oblivious adversary corrupts the vector of labels $\vec y \in \R^m$
by adding an $\eta m$-sparse corruption vector $\vec b$.
Since $\vec b$ is independent of the covariates, oblivious noise can be viewed as
corrupting a sample with probability $\eta$ with a random non-zero entry of $\vec b$.
Consequently, oblivious noise can be seen as a special case of Massart noise. 
We formally compare these two noise models in more detail in Appendix~\ref{supp:oblivious}.
A line of work~\cite{BhatiaJKK17,SuggalaBR019,d2020regress,pesme2020online}
studied robust linear regression under oblivious noise and developed efficient exact recovery 
algorithms under strong distributional assumptions.

\section{Warm-up: Linear Regression with Massart Noise}
\label{sec:linear}

To establish our algorithmic result for linear regression, 
we establish structural conditions under which we can perform efficient 
$\ell_0$-minimization for linear functions under Massart noise. 
It is imperative that we find the $\ell_0$-minimizer with respect to $\vec{w}$ since, 
with a sufficient number of samples, the $\ell_0$-minimizer is the true function we wish to recover. 

In Section~\ref{ssec:special_case}, we show that if a radial-isotropic transformation $\vec A$ exists 
for all samples, then appropriately transforming the data $(\x_i, y_i)$ 
to $(\tilde{\x}_i, \tilde{y}_i)$ via $\vec A$ and subsequently solving
for the empirical $\ell_1$-loss 
$\arg \min_{\vec{w} \in \R^d} \frac{1}{m}\littlesum_{i=1}^m |\tilde{y}_i - \vec{w} \cdot \tilde{\x}_i|$ 
can efficiently recover the true parameter $\vec{w}^*$.
However, such a radial-isotropic transformation may not exist. 
We handle this general case in Section~\ref{ssec:general_case} by leveraging the idea from 
Section~\ref{ssec:special_case} recursively on a subset of the samples.

\subsection{Special Case: Zero Mass on Linear Subspaces}
\label{ssec:special_case}
We first consider the case where there is zero probability mass on
any linear subspace for the marginal distribution $\D_\x$. 
That is, we assume that the parameter $\rho$ from Theorem~\ref{thm:linear_exact} is set to one, 
so that for any finite set of samples $(\x_i, y_i)_{i=1}^{m} \subset \R^d \times \R$,
the examples $\x_i$'s are in general position,
i.e., every set of $d$ examples is linearly independent. 
Under this assumption, we prove the following special case 
of Theorem~\ref{thm:linear_exact} (corresponding to $\rho = 1$).

\begin{theorem}[Special case of Theorem~\ref{thm:linear_exact}] \label{thm:linear_special}
Let $\D_{\x}$ be a distribution on $\R^d$ that has zero measure on any linear subspace 
and let $\eta < 1/2$ be the upper bound on the Massart noise rate. Denote $\vec{w}^*$ the vector 
representing the true linear function. There is an algorithm that draws $\tilde{O}(\frac{d^3}{(1-2\eta)^2})$ samples, 
runs in $\poly(d,b,(1-2\eta)^{-1})$ time, where $b$ is an upper bound on the bit complexity of the samples and parameters, 
and outputs $\vec{w}^*$ with probability at least $9/10$.
\end{theorem}

The algorithm for recovering linear functions in this case is given in pseudocode below. 

\begin{algorithm}[hbt!]
   \caption{Linear function recovery via radial isotropy}
   \label{alg:linear}
\begin{algorithmic}[1]
   \State Draw $m = \tilde{O}(\frac{d^3}{(1 - 2\eta)^2})$ samples $(\x_i, y_i)_{i=1}^{m}$ with $\eta$-Massart noise
   \State Compute $\vec{A}$ that puts $(\x_i, y_i)_{i=1}^{m}$ in $1/2$-approximate radial-isotropic position
   \State Compute $\hat{\vec{w}} \gets \arg \min_{\vec{w} \in \R^d} \littlesum_{i=1}^m |\frac{ y_i}{\|\vec{A} \x_i\|_2} - \vec{w} \cdot \frac{\vec{A} \x_i}{\|\vec{A} \x_i\|_2}|$ by solving an LP
   \State {\bfseries return}  $\vec{A} \hat{\vec{w}}$
\end{algorithmic}
\end{algorithm}

In fact, there is no need to compute an exact radial-isotropic transformation ($\gamma = 0$),
as an approximate one suffices. An approximate radial-isotropic transformation can be computed 
efficiently, see, e.g.,~\cite{HardtM13,artstein2020radial}, as stated in the following lemma. 

\begin{lemma}\label{lemma:forster_approx}
Given $S \subset \R^d$ in general position, there is a $\poly(n, d, b, \gamma^{-1})$ time algorithm that computes a positive definite symmetric matrix $\vec{A}$ such that $\big\{\frac{\vec{A} \x}{\|\vec{A}\x\|}: \x \in S \big\}$ is in $\gamma$-approximate radial-isotropic position, where $b$ is an upper bound on the bit complexity 
of the parameters and samples in $S$. 
Morever, the condition number of $\vec{A}$ is at most $2^{\poly(n,d,b)}$.
\end{lemma}

\noindent For completeness, we provide a proof of this lemma in Appendix~\ref{supp:linear}.

Since computing such an approximate transformation $\vec{A}$ and solving a linear program (LP) 
can be done efficiently, Algorithm \ref{alg:linear} gives a polynomial runtime 
for the case that the examples are in general position. It remains to prove correctness.

The proof of Theorem \ref{thm:linear_special} relies on two key ideas.
The first idea is that under some structural conditions about the given samples,
the $\ell_1$-loss minimizer is identical to the $\ell_0$-loss minimizer.
These conditions are presented in Lemma \ref{lemma:structural_condition} below.
The second idea is that any (sufficiently large) set of samples in radial isotropic position
guarantees that the structural conditions of Lemma \ref{lemma:structural_condition} hold
with high probability over the adversarial corruptions. Such a transformation can be applied
to any set of points in general position without loss of generality.

\begin{lemma}[Structural Condition for Recovery]
\label{lemma:structural_condition}
Given $f: \R \rightarrow \R$ and $m$ samples $(\x_i, y_i)_{i=1}^m$ in $\R^d$, let the $\ell_0$-minimizer 
$\vec{w}^* = \arg \min_{\vec w \in \R^d} \frac{1}{m}\littlesum_{i=1}^m \|y_i - f(\vec w \cdot \x_i)\|_0$ 
be unique. If 
\begin{equation}
\label{eq:structural}
\sum_{y_i = f(\vec w^* \cdot \x_i)} |f((\vec w^* + \vec r)\cdot \x_i) - f(\vec w^* \cdot \x_i)| > \sum_{y_i \neq f(\vec w^* \cdot \x_i)} |f((\vec w^* + \vec r)\cdot \x_i) - f(\vec w^* \cdot \x_i)| \tag{$\star$}
\end{equation}
for all non-zero $\vec r \in \R^d$, then $\vec{w}^*$ is also the $\ell_1$-minimizer 
$\arg \min_{\vec w \in \R^d} \frac{1}{m}\littlesum_{i=1}^m |y_i - h(\x_i)|$.
\end{lemma}

\begin{proof}
Let $\vec{w}^*$ be the parameter corresponding to the $\ell_0$-minimizer. Denote the $\ell_1$-loss $\hat{L}(\vec{w}) = \frac{1}{m}\littlesum_{i=1}^m |y_i - f(\vec{w}\cdot \x_i)|$. Given the strict inequality in (\ref{eq:structural}), for non-zero $r\in \R^d$, we have that
\begin{align*}
&\;\;\;\;\; m\big(\hat{L}(\vec{w}^*+\vec{r}) - \hat{L}(\vec{w}^*) \big) \\ 
&= \littlesum_{i=1}^m |f((\vec{w}^*+\vec{r})\cdot \x_i) - y_i| - \littlesum_{i=1}^m |f(\vec{w}^*\cdot \x_i) - y_i| \\
&= \littlesum_{y_i = f(\vec{w}^* \cdot \x_i)} |f((\vec{w}^*+\vec{r})\cdot \x_i) - f(\vec{w}^*\cdot \x_i)| + \littlesum_{y_i \neq f(\vec{w}^* \cdot \x_i)} |f((\vec{w}^*+\vec{r})\cdot \x_i) - y_i| - \littlesum_{y_i \neq h_{\vec{w}^*}(\x_i)} |f(\vec{w}^*\cdot \x_i) - y_i| \\
&\ge  \littlesum_{y_i = f(\vec{w}^* \cdot \x_i)} |f((\vec{w}^*+\vec{r})\cdot \x_i) - f(\vec{w}^*\cdot \x_i)| - \littlesum_{y_i \neq f(\vec{w}^* \cdot \x_i)}  |f((\vec{w}^*+\vec{r})\cdot \x_i) - f(\vec{w}^*\cdot \x_i)| > 0.
\end{align*}
Therefore, the $\ell_0$-minimizer $\vec{w}^*$ is also the $\ell_1$-minimizer $\arg \min_{\vec w \in \R^d} \frac{1}{m}\littlesum_{i=1}^m |y_i - h(\x_i)|$.
\end{proof}

The structural condition (\ref{eq:structural}) for linear functions reduces to having the sum of $|\vec{r}\cdot \x_i|$ for the ``good'' points be greater than the sum of $|\vec{r}\cdot \x_i|$ for the ``bad'' points in every direction $\vec{r}$. However, this implies that if one sample is much greater in norm than the others in some direction, this point can have undue influence and may easily dominate the $\ell_1$-loss. Therefore, without any preprocessing or transformation to the data, one has to rely on naively increasing the sample complexity until there are enough points in this direction to satisfy condition (\ref{eq:structural}). Instead, we minimize the dominating effects of such outlier points and reduce the sample complexity through transforming the dataset with radial isotropy.

Given Lemma~\ref{lemma:forster_approx} and \ref{lemma:structural_condition}, we now prove the main result for robust linear regression when $\rho = 1$ based on Algorithm~\ref{alg:linear}. 

\begin{proof}[Proof of Theorem~\ref{thm:linear_special}]
Without loss of generality, assume $\x_i$'s are unit vectors. The linear function can be written as follows
\[y = \vec{w}^* \cdot \x = (\vec{A}^{-1} \vec{w}^*) \cdot (\vec{A} \x) \;, \]
where $A$ denotes the $\gamma$-approximate radial-isotropic transformation where $\gamma = 1/2$. This means that the solution to the LP in Algorithm \ref{alg:linear} returns $\hat{\vec{w}} = \vec A^{-1} \vec{w}^*$ given Lemma \ref{lemma:structural_condition} is satisfied. Therefore, we output $\vec{A} \vec{\hat{w}}$ as the true direction of the original dataset.

The rest of the proof establishes that the structural condition holds. 
By radial isotropy, for $\vec{r} \in \R^d$, we have that
\begin{align*}
\frac{1}{m}\sum_{i=1}^m |\vec{r}\cdot \tilde{\x}_i| 
\ge \frac{\|\vec{r}\|_2}{m} \cdot \sum_{i=1}^m  (\frac{\vec{r}}{\|\vec{r}\|_2}\cdot \tilde{\x}_i)^2 
\ge \frac{(1-\gamma)\|\vec{r}\|_2}{d} \;.
\end{align*}
Define $\tilde{\D}_\x$ on the $d$-dimensional unit sphere $\mathcal{S}^{d-1}$ 
to be the distribution $\D_\x$ after the transformation $\x \mapsto \frac{\vec{A} \x}{\|\vec{A} \x\|_2}$. 
We use the following standard VC inequality.
\begin{lemma}[VC Inequality] \label{lemma:VC_ineq}
Let $\nu$ be a probability measure and $\mathcal{A}$ be a family of sets of VC dimension $d$.
For any $\epsilon, \delta > 0$,
with $m = O((d + \ln(1/\delta))/\epsilon^2)$ samples $\x_1, \x_2, \dots, \x_m$ from $\nu$, we have
\[\Pr\left[\sup_{A \in \mathcal{A}} |\nu(A) - \nu_m(A)| > \epsilon \right] \le \delta \;,\]
where $\nu(A) = \Pr[\x_1 \in A]$ and $\nu_m(A) = (1/m)\littlesum_{i=1}^m \mathds{1}\{\x_i \in A\}$.
\end{lemma} 
By Lemma~\ref{lemma:VC_ineq}, 
with $O(\frac{d}{\epsilon^2})$ samples, with high probability the following holds
\[\sup_{\vec{r} \in \R^d} \left|\Pr_{\tilde{\x} \sim \tilde{\D}_\x}[|\vec{r}\cdot \tilde{\x}| > t] - \frac{1}{m} \sum_{i=1}^{m} \mathds{1} \{|\vec{r}\cdot \tilde{\x}_i| > t\} \right| \le \epsilon \;,\]
since the VC dimension of the set 
$\mathcal{F} = \{\mathds{1}_{\{|\vec{r}\cdot \x| > t\}}: \vec{r} \in \R^d, t \in \R\}$ is $O(d)$. 
By integration, we get that
\begin{align*}
\frac{1}{m}\sum_{i=1}^m |\vec{r}\cdot \tilde{\x}_i| &= \int_{0}^{\infty} \left( \frac{1}{m} \sum_{i=1}^{m} \mathds{1} \{|\vec{r}\cdot \tilde{\x}_i| > t\} \right) dt \\
&= \int_{0}^{\max_i |\vec{r} \cdot \tilde{\x}_i|} \left( \frac{1}{m} \sum_{i=1}^{m} \mathds{1} \{|\vec{r}\cdot \tilde{\x}_i| > t\} \right) dt \\
&\le \E_{\tilde{\x} \sim \tilde{\D}_\x}[|\vec{r}\cdot \tilde{\x}|] +  \epsilon\cdot \max_i |\vec{r}\cdot \x_i| \\
&\le \E_{\tilde{\x} \sim \tilde{\D}_\x}[|\vec{r}\cdot \tilde{\x}|] + \frac{\epsilon d}{1-\gamma} \left( \frac{1}{m}\sum_{i=1}^m |\vec{r}\cdot \tilde{\x}_i|  \right) \;,
\end{align*}
since $\max_i |\vec{r}\cdot \tilde{\x}_i| \le \|\vec{r}\|_2$. 
Then we have the inequality:
\[\|\vec{r}\|_2 \le \frac{(1-\gamma)d}{1-\gamma-\epsilon d} \E_{\tilde{\x} \sim \tilde{\D}_\x}[|\vec{r}\cdot \tilde{\x}|].\]
We can now get a lower bound for the uncorrupted samples.
\begin{align*}
\frac{1}{m}\sum_{i=1}^{m} |\vec{r}\cdot \tilde{\x}_i|\mathds{1} \{y_i = \vec{w}^* \cdot \x_i\} &= \int_{0}^{\infty} \left( \frac{1}{m} \sum_{i=1}^{m} \mathds{1} \{|\vec{r}\cdot \tilde{\x}_i| > t \wedge y_i=\vec{w}^* \cdot \x_i\} \right) dt \\
&\ge \E_{\tilde{\x} \sim \tilde{\D}_\x}[|\vec{r}\cdot \tilde{\x}| \mathds{1} \{y = \vec{w}^* \cdot \x\} ] - \epsilon \max_{\tilde{x} \in \mathrm{supp}(\D_\x)} |\vec{r}\cdot \tilde{x}| \\
&\ge (1-\eta)\E_{\tilde{\x} \sim \tilde{\D}_\x}[|\vec{r}\cdot \tilde{\x}|] - \frac{(1-\gamma)\epsilon d}{1-\gamma-\epsilon d} \E_{\tilde{\x} \sim \tilde{\D}_\x}[|\vec{r}\cdot \tilde{\x}|].
\end{align*}
We similarly obtain an upper bound for the corrupted samples of $y_i \neq \vec{w}^* \cdot \x_i$, 
so by setting $\epsilon = O(\frac{1-2\eta}{d})$, with $m = \tilde{O}(\frac{d^3}{(1-2\eta)^2})$ samples, 
the structural condition of Lemma $\ref{lemma:structural_condition}$ is satisfied 
for any non-zero $\vec{r} \in \R^d$ with high probability. 
Thus, with Lemma \ref{lemma:forster_approx}, this proves Theorem \ref{thm:linear_special}.
\end{proof}

\subsection{The General Case: Proof of Theorem~\ref{thm:linear_exact}}
\label{ssec:general_case}

In general, we assume that $\D_\x$ is a distribution supported on $b$-bit integers such that 
$\Pr_{\x \sim \D_{\x}}[\vec{r} \cdot \x = 0] \le 1 - \rho$, for all non-zero $\vec{r} \in \R^d$, 
where $\rho \in (0, 1]$ is a parameter.
Since a non-trivial fraction of samples may concentrate on a particular subspace, 
there may not exist a transformation that puts the points into radial-isotropic position.
In fact, the following condition is necessary and sufficient for the existence of such a transformation.

\begin{lemma}[Lemma 4.19 of \cite{hopkins2020point}] \label{lemma:kane_radial}
Given a set of points $S \subseteq \R^d$, the following conditions are equivalent:
\begin{enumerate}
\item For any $\gamma > 0$, there exists an invertible linear transformation $\vec A$ 
such that $\vec A$ puts $S$ in $\gamma$-approximate radial-isotropic position.
\item For every $1 \le k \le d$, every $k$-dimensional subspace contains at most $k/d$-fraction of $S$.
\end{enumerate}
\end{lemma}

Given the condition above, there does not exist a radial-isotropic transformation for all non-zero points
if there exists a $k$-dimensional subspace $V$ that contains more than $k/d$-fraction of the non-zero points. 
In this case, we use the following algorithmic result from \cite{DKT21-forster} that efficiently computes
a radial-isotropic transformation for the points that lie on the subspace $V$.

\begin{lemma}[Theorem 1.4 of \cite{DKT21-forster}] \label{lemma:dkt_forster}
There exists an algorithm that, given
a set $S$ of $n$ points in $\mathbb{Z}^d\setminus \{0\}$ of bit complexity at most $b$ and $\delta > 0$, 
runs in $\poly(n, d, b, \log(1/\delta))$ time, and returns a subspace $V$ of $\R^d$ containing 
at least a $\dim(V)/d$-fraction of the points in $S$ and a linear transformation $\vec A : V \rightarrow V$ such that 
$\frac{1}{|S \cap V|} \littlesum_{\x \in S \cap V} (\frac{\vec A \x}{\|\vec A \x\|_2})(\frac{\vec A \x}{\|\vec A \x\|_2})^T 
= (1/\dim(V)) I_V + O(\delta)$, where the error is in spectral norm.
\end{lemma}

This algorithmic result relaxes the assumption on the underlying distribution of Theorem~\ref{thm:linear_special} 
by allowing us to compute a radial-isotropic transformation for a set of points that may concentrate on a particular subspace.

Our ReLU learning algorithm leverages this algorithmic result. 
The main algorithmic idea is to apply radial-isotropic transformation iteratively on any concentrated subspace.
For example, if there exists a subset of points lying in a $k$-dimensional subspace $V$, so that there does not
exist a radial-isotropic transformation for the whole set of points, i.e., more than $k/d$-fraction of the points lie on $V$, 
then we can efficiently find such a subspace $V$ with a corresponding radial-isotropic transformation 
in its lower-dimensional space, using Theorem~\ref{lemma:dkt_forster}. With this ingredient, 
we can compute $\proj_V \vec w^*$ using Algorithm~\ref{alg:linear} in $k$-dimensions. 
Similarly, we compute the orthogonal component of $\vec w^*$ on the orthogononal subspace $V^\perp$. 
Here it is important that we have enough points from the original set of samples
that do not project to zero in $V^\perp$, since a significant portion lies on $V$. 

The pseudocode of our learner is presented below, followed by a statement and proof of its properties.
We denote by \textsc{GeneralizedForster} the algorithm that achieves Theorem 1.4 of \cite{DKT21-forster}.

\begin{algorithm}[hbt!]
   \caption{Linear function recovery via radial isotropy}
\begin{algorithmic}[1]
    \Procedure{$\textsc{RecoverLinear}$}{$(\x_i, y_i)_{i=1}^{m} \subset \R^d \times \R$}
    \State Run \textsc{GeneralizedForster} to find subspace $V$ and radial-isotropic transformation $\vec A$
    \If{$\dim(V) = d$} 
      \State $\tilde{S} \gets \{(\frac{\vec{A} \x_i}{\|\vec{A} \x_i\|_2}, \frac{y_i}{\|\vec{A} \x_i\|_2}): i \in [m] \text{ for } \x_i \neq 0\}$
      \State $\tilde{\vec{w}} \gets \arg \min_{\vec{w} \in \R^d} \sum_{(\tilde{\x}, \tilde{y}) \in \tilde{S}} |\tilde{y} - \vec{w} \cdot \tilde{\x}|$ by solving the LP.
      \State \textbf{return} $\vec{A} \tilde{\vec{w}}$
    \EndIf
      \State $S_V \gets \{(\x_i, y_i): i \in [m] \text{ where } \x_i \in V\}$
      \State Rotate $S_V$ into $\R^k$ and run Algorithm~\ref{alg:linear} with transformation $\vec A$.
      \State Let $\vec{w}$ be the output from the previous step, rotated back into $\R^d$.
      \State Let $V^{\perp}$ be the orthogonal subspace to $V$ in $\R^d$.
      \State $S_V^\perp \gets \{(\proj_{V^\perp} \x_i,\ y_i - \vec{w} \cdot \proj_{V} \x_i): i \in [m] \text{ where } \x_i \notin V\}$
      \State Rotate $S_V^\perp$ into $\R^{d-k}$ and run $\textsc{RecoverLinear}$ and rotate back to compute $\vec{w}^\perp \in \R^d$.
      \State \textbf{return} $\vec{w} + \vec{w}^\perp$
    
    \EndProcedure

    \State $m \gets \tilde{O}(\frac{d^3}{\rho(1 - 2\eta)^2})$
    \State Draw $m$ i.i.d.\ samples $(\x_i, y_i)_{i=1}^{m}$ with $\eta$-Massart noise
    \State $\textsc{RecoverLinear}((\x_i, y_i)_{i=1}^{m})$
  \end{algorithmic}
\end{algorithm}

\begin{proof}[Proof of Theorem~\ref{thm:linear_exact}]
Assume, for the sake of simplicity, that $\rho = 1$ so that the distribution does not concentrate 
on any lower-dimensional subspace. Then there always exists a radial-isotropic transformation $\vec A$ 
for any set of samples, as long as it has at least $d$ points, since all points are in general position. 
As we have shown in the proof of Theorem~\ref{thm:linear_special}, when $\rho=1$, 
the algorithm correctly returns $\vec{w}^*$ with high probability using $\tilde{O}(\frac{d^3}{(1-2\eta)^2})$ samples.

For $0 < \rho < 1$, the correctness of the algorithm follows from a standard divide-and-conquer argument, 
as long as each call to the algorithm is supplied with a sufficient number of (non-zero) samples. 
Thus, we only need to analyze the sample complexity and ensure each recursive call into $k$-dimensions 
receives enough samples as an input. 

For the first iteration of $\textsc{RecoverLinear}$ in $\R^d$, if there exists $\vec A$
that puts the remaining non-zero points 
into radial isotropy, we only need to sample $\tilde{O}(\frac{d^3}{\rho(1-2\eta)^2})$ points from $\D_\x$. 
The factor of $\rho^{-1}$ appears because in the worst case we have $\rho$-fraction
of the marginal distribution $\D_\x$ concentrating on $0$, so that $\rho$-fraction of the samples
cannot be put into radial-isotropic position. 
Thus, we need $m_d := \tilde{O}(\frac{d^3}{\rho(1-2\eta)^2})$ many samples for $d$ dimensions 
if $\vec{A}$ exists.
Then, similarly to Theorem~\ref{thm:linear_special}, if $\vec{A}$ exists, 
$m_d$ many samples are sufficient for $\textsc{RecoverLinear}$ in $\R^d$ to find $\vec{w}^*$ 
with probability at least $9/10$. We now need to prove that the algorithm works with $m_d$ samples 
with probability at least $9/10$ even when $\vec{A}$ does not exist.

In the case that $\vec{A}$ does not exist, by Lemma~\ref{lemma:kane_radial}, 
there must exist a $k$-dimensional subspace $V$ that contains more than $k/d$-fraction of the points. 
Here, we apply the algorithm on the subset $S_V := \{(\x_i, y_i): i \in [m] \text{ where } \x_i \in V\}$ in $\R^k$. 
In this subproblem, the number of samples is $|S_V| \ge (k / d) m_d \ge m_k$, 
and thus is sufficient to accurately compute the projection of $\vec{w}^*$ on $V$.

What remains is ensuring that $S_V^\perp$ has enough non-zero samples, 
despite more than $k/d$-fraction of the points projecting to zero on the orthogonal subspace $V^\perp$. 
In other words, we want to upper bound the probability that $S_V$ simultaneously contains more than 
$k/d$-fraction of the points and more than $m_d -\rho m_{d-k}$ points, for $1 \le k < d$. 
By the union bound, we can simplify the following expression.
\begin{align*}
&\;\;\;\; \Pr \left[ \big( |S_V| \ge km_d/d \big) \wedge \big( |S_V| \ge m_d-\rho m_{d-k}\big) \text{ for } 1 \le k < d \right] \\ 
&\le \Pr[|S_V| \ge (1-\rho+ \rho k / d)m_d \text{ for } 1 \le k < d  ] \\
&\le \sum_{k=1}^{d-1} \Pr[|S_V| \ge (1-\rho+ \rho k / d)m_d ]
\le (d-1) \Pr[|S_V| \ge (1-\rho+\rho/d)m_d ] \;.
\end{align*}
If $\rho \ge 1/2$, Hoeffding's inequality bounds from above this quantity 
by $(d-1) \exp(-\frac{m_d}{2d^2})$. 
In the case that $1-\rho \ge 1/2$, we have that $D_{\mathrm{KL}}(1-\rho + \delta || 1 - \rho) \ge \frac{\delta^2}{2\rho(1-\rho)}$, 
so the Chernoff bound yields the following inequality: 
\begin{align*}
\sum_{k=1}^{d-1} \Pr[|S_V| \ge (1-\rho+\frac{\rho k}{d})m_d] 
&\le \sum_{k=1}^{d-1} \exp \left( -\frac{(\rho k / d)^2}{2\rho(1-\rho)}m_d \right) \\
&\le (d-1) \exp \left( -\frac{\rho}{2d^2}m_d \right) \;.
\end{align*}
Thus, with $m_d  = \tilde{O}(\frac{d^3}{\rho(1-2\eta)^2})$, we can guarantee that any heavy subspace $V$ 
with more than $k/d$-fraction of the points will not contain too many samples, 
meaning that there will be $\tilde{O}(\frac{d^3}{(1-2\eta)^2})$ non-zero points in $S_V^\perp$ 
to compute a radial-isotropic transformation if one exists. 
Furthermore, the error probability we calculated above may accumulate over at most $d$ recursive calls. 
Since the error we have above is bounded in terms of $\exp(-\frac{d^3}{(1-2\eta)^2})$, 
after applying the union bound, we can still ensure that the algorithm finds $\vec{w}^*$ with high probability.
\end{proof}

\section{ReLU Regression with Massart Noise}
\label{sec:relu}

In this section, we give our main algorithmic result of exact recovery for 
ReLUs in the presence of Massart noise, establishing Theorem~\ref{thm:relu_exact}.

For the case of ReLUs, we can still use the structural condition of Lemma 
$\ref{lemma:structural_condition}$ connecting $\ell_1$-minimization 
$\arg \min_{\vec{w}} \frac{1}{m}\sum_{i=1}^m |y_i - \relu(\vec{w} \cdot \x_i)|$ 
to $\ell_0$-minimization. 
However, efficiently minimizing this $\ell_1$-objective is no longer straightforward, 
because the objective function is non-convex. 
Despite this fact, we show that it is possible to efficiently recover a ReLU under mild 
anti-concentration assumptions on the underlying distribution.

The key idea enabling the algorithm of Theorem \ref{thm:relu_exact} is characterizing 
the condition under which we can compute an efficient separation oracle 
between the query $\vec{w}$ and the true parameter vector $\vec{w}^*$. 
Once we obtain a separation oracle, we can use the ellipsoid method 
to recover $\vec{w}^*$ exactly. In turn, similarly to Lemma $\ref{lemma:structural_condition}$, 
we identify a sufficient structural condition on the dataset, 
which allows us to efficiently compute a separating hyperplane between 
$\vec{w}$ and $\vec{w}^*$ if $\vec{w} \neq \vec{w}^*$, 
and then use radial-isotropic transformations such that this condition is satisfied. 
We state this separation condition in the following lemma.

\begin{lemma}[Separation Condition]\label{lemma:separation_condition}
Let $\mathcal{H}$ be a hypothesis class such that 
$\mathcal{H} = \{h_{\vec{w}}: h_{\vec{w}}(\x) = f(\vec{w} \cdot \x), \vec{w} \in \R^d\}$, 
where $f: \R \rightarrow \R$ is monotonically non-decreasing. 
Given a set of $m$ samples $(\x_i, y_i)_{i=1}^m$, let 
$\vec{w}^* = \arg \min_{\vec{w} \in \R^d} \frac{1}{m}\sum_{i=1}^m \|y_i - f(\vec{w}\cdot \x_i)\|_0$ be unique. Let $\Delta > 0$ and $\mathcal{B}(\vec{w}, \Delta)$ be the open ball of radius $\Delta$ centered at $\vec{w}$. Denote the empirical $\ell_1$-loss $\hat{L}(\vec{w}) = (1/m)\sum_{i=1}^m |y_i - f(\vec{w}\cdot \x_i)|$. 
Given a query $\vec{w}_0 \notin \mathcal{B}(\vec{w}^*, \Delta)$, if 
\begin{equation}
\label{eq:sep}
\sum_{y_i = f(\vec{w}^* \cdot \x_i)} |(\vec{w}_0 - \vec{w}^*)\cdot \x_i| f'(\vec{w}_0\cdot \x_i) - \sum_{y_i \neq f(\vec{w}^* \cdot \x_i)} |(\vec{w}_0 - \vec{w}^*)\cdot \x_i| f'(\vec{w}_0\cdot \x_i) \ge \Delta m \;, \tag{$\dagger$}
\end{equation}
then $\nabla \hat{L}(\vec{w}_0) \cdot (\vec w_0  - \vec w) = 0$ is a separating hyperplane for $\vec{w}_0$ and $\mathcal{B}(\vec{w}^*, \Delta/2)$ such that $\nabla \hat{L}(\vec{w}_0) \cdot (\vec w_0  - \vec w) > 0$ for $\vec w \in \mathcal{B}(\vec{w}^*, \Delta/2)$.
\end{lemma}

\begin{proof}
We establish that we can find a separating hyperplane that separates $\vec{w}$ 
sufficiently far from $\vec{w}^*$. This guarantees that the ellipsoid method shrinks in volume, 
while always containing a small ball around $\vec{w}^*$ that is never cut by a separating hyperplane.

Define the empirical loss $\hat{L}(\vec{w}) = \frac{1}{m}\sum_{i=1}^m |y_i - f(\vec{w}\cdot \x_i)|$. 
We can write
\begin{align*}
(\vec{w}_0-\vec{w}^*)\cdot \nabla \hat{L}(\vec{w}_0) &= \frac{1}{m}\sum_{i=1}^m [\sign(f(\vec{w}_0\cdot \x_i) - y_i) (\vec{w}_0 - \vec{w}^*) \cdot \nabla_{\vec{w}} f(\vec{w}_0\cdot \x_i) ] \\
&= \frac{1}{m}\sum_{i=1}^m [\sign(f(\vec{w}_0\cdot \x_i) - y_i) (\vec{w}_0 - \vec{w}^*) \cdot \x_i f'(\vec{w}_0\cdot \x_i) ] \\
&= \frac{1}{m}\sum_{i: y_i = f(\vec{w}^*\cdot \x_i)} [(\vec{w}_0 - \vec{w}^*) \cdot \x_i \cdot \sign(\vec{w}_0\cdot \x_i - \vec{w}^*\cdot \x_i) f'(\vec{w}_0\cdot \x_i)] \\
& \ \ \ \ + \frac{1}{m}\sum_{i: y_i \neq f(\vec{w}^*\cdot \x_i)} [(\vec{w}_0 - \vec{w}^*) \cdot \x_i \cdot \sign(f(\vec{w}_0\cdot \x_i) - y_i) f'(\vec{w}_0\cdot \x_i)] \\
&\ge \frac{1}{m}\sum_{i: y_i = f(\vec{w}^*\cdot \x_i)} [|(\vec{w}_0 - \vec{w}^*) \cdot \x_i| f'(\vec{w}_0\cdot \x_i)] \\
& \ \ \ \ - \frac{1}{m}\sum_{i: y_i \neq f(\vec{w}^*\cdot \x_i)} [|(\vec{w}_0 - \vec{w}^*) \cdot \x_i| f'(\vec{w}_0\cdot \x_i)] \ge \Delta \;,
\end{align*}
where the third equality follows from monotonicity of $f$ and the last inequality 
follows from the inequality condition (\ref{eq:sep}) on the set of samples.
\end{proof}

In particular, the gradient of the empirical $\ell_1$-loss gives us the separating hyperplane above.
Other than the fact that only the points in the nonnegative side of the halfspace $\vec{w}$ 
are considered in the separation condition (\ref{eq:sep}), the condition resembles 
the structural condition used for linear functions. Analogously, we apply a radial-isotropic 
transformation to the points of $\vec{w}\cdot \x_i \ge 0$ and iterate the procedure 
on a concentrated subspace if such transformation does not exist for all non-zero points. 
We thus obtain a sub-procedure of the ellipsoid method (Algorithm~\ref{alg:new_sep}).
We specify the radius of a ball $\Delta$ in the proof of Theorem~\ref{thm:relu_exact},
but assume $\Delta$ is a quantity smaller than half of the distance between any two rational points
(which depends on the bit complexity $b$ of the samples and parameter).

\begin{algorithm}[hbt!]
   \caption{Separation oracle sub-procedure}
   \label{alg:new_sep}
\begin{algorithmic}[1]
   \State {\bfseries Input: } $\{(\x_i, y_i)\}_{i=1}^m$ with Massart noise and query $\vec{w}_0$.
   \State {\bfseries Output: } If $\vec{w}_0 \in  \mathcal{B} (\vec{w}^*, \Delta)$, return ``Yes''.
   \State \hspace{17.6mm} If not, return a separating hyperplane between $\vec{w}_0$ and $\mathcal{B} (\vec{w}^*, \Delta/2)$. 
   \Procedure{$\textsc{SEP}$}{$(\x_i, y_i)_{i=1}^m, \vec{w}_0$}
   \If{$\relu(\vec{w}_0 \cdot \x)$ fits at least $\frac{m}{2}$ points}
     \State {\bfseries return} ``Yes''
   \EndIf
   \State Define $S = \{(\x_i, y_i): \vec{w}_0 \cdot \x_i \ge 0, \x_i \neq 0 \text{ for } i \in [m]\}$.
   \State Run \textsc{GeneralizedForster} on $S_\x$ to find subspace $V$ and radial-isotropic transformation $\vec A$
   \If{$\dim(V) = d$} 
     \State $\vec{r} = \frac{1}{|S|}\sum_{(\x_i, y_i) \in S} \frac{\vec{A} \x_i}{\|\vec{A} \x_i\|_2} \cdot \sign{(\vec{w}_0 \cdot \x_i - y_i)}$
     \State {\bfseries return} separating hyperplane $\vec A^{-1} \vec{r} \cdot (\vec{w}_0 - \vec w) = 0$ 
   \EndIf
   \State Rotate $\{(\x, y) \in S: \x \in V\}$ and the query $\proj_V \vec{w}_0$ into $\R^k$ and run $\textsc{SEP}$ on them
   \If{\textsc{SEP} returns a hyperplane in $\R^k$} 
     \State {\bfseries return} the hyperplane rotated back and extended into $\R^d$ so that it is orthogonal to $V$
   \EndIf
   \State Let $V^\perp$ be the orthogonal subspace to $V$ in $\R^d$.
   \State Define $S_V^\perp = \{(\proj_{V^\perp} \x,\ y - \proj_{V} \vec{w}_0 \cdot \proj_{V} \x) : (\x, y) \in S \setminus V\}$.
   \State Rotate $S_V^\perp$ and the query $\proj_{V^\perp} \vec{w}_0$ into $\R^{d-k}$ and run $\textsc{SEP}$ on them
   \If{\textsc{SEP} returns a hyperplane in $\R^{d-k}$} 
     \State {\bfseries return} the hyperplane rotated back and extended into $\R^d$ so that it is orthogonal to $V^\perp$
   \EndIf
   \State {\bfseries return} ``Yes''
   \EndProcedure
\end{algorithmic}
\end{algorithm}

The main difference between the algorithm for ReLUs and linear functions is that here we must apply a different radial-isotropic transformation to every new subset of points in every iteration, depending on the query $\vec{w}_0$. In turn, the algorithm transforms the space according to 
a new transformation $\vec{A}$, computes a separating hyperplane, 
and transforms the hyperplane back into the original space. 
Due to these repeated transformations, the proof of Theorem \ref{thm:relu_exact} 
requires a more intricate argument to make the ellipsoid method work correctly. 
We now prove Theorem~\ref{thm:relu_exact}.

\begin{proof}[Proof of Theorem~\ref{thm:relu_exact}]
Let $\vec{w}_0$ be the original query to the oracle and assume that 
the separation condition (\ref{eq:sep}) holds for a set of points 
$({\vec{A} \x_i}/{\|\vec{A} \x_i\|_2}, {y_i}/{\|\vec{A} \x_i\|_2})_{i=1}^m$ and $\vec{A}^{-1}\vec{w}_0$. 
Then, by Lemma \ref{lemma:separation_condition}, we have that 
$\vec{r} \cdot (\vec{A}^{-1}\vec w_{0} - \vec w) = 0$, 
where $\vec r = (1/m)\sum_{i=1}^m (\vec{A} \x_i / \|\vec{A} \x_i\|_2) \cdot \sign{(\vec{w}_0 \cdot \x_i - y_i)}$ separates $\vec{A}^{-1}\vec{w}_0$ and $\vec{A}^{-1}\vec{w}^*$.
Thus, the separation for $\vec{w}_0$ and $\vec{w}^*$ is 
$\vec{A}^{-1}\vec{r} \cdot (\vec{w}_0 - \vec{w}) = 0$.

It remains to check the sample complexity necessary 
for the separation condition (\ref{eq:sep}) to hold. 
Similarly to the proof of Theorem~\ref{thm:linear_exact}, 
we analyze two cases of $\rho = 1$ and $0 < \rho < 1$.

\paragraph*{Special Case: $\rho = 1$\\}
We first assume that $\rho = 1$, so that there is zero measure 
on any linear subspace and any finite set of $\x_i$'s is in general position.

Each unique set of $\{(\x_i, y_i): \vec{w}_0 \cdot \x_i \ge  0 \text{ for } i \in [m]\}$ 
determines a radial isotropic transformation, 
but there can only be at most $m^{d+1}$ unique sets, by the VC-dimension of halfspaces. 
So, there are only at most $m^{d+1}$ radial-isotropic transformations we have to consider. 
Let $\vec{A}$ be the linear transformation of the radial-isotropic transformation 
applied to points of $\vec{w}_0 \cdot \x_i \ge 0$. 
Denote $(\tilde{\x}_i, \tilde{y}_i) = ({\vec{A} \x_i}/{\|\vec{A} \x_i\|_2}, {y_i}/{\|\vec{A} \x_i\|_2})$, 
$\tilde{\vec{w}}^* = \vec{A}^{-1} \vec{w}^*$, $\tilde{\vec{w}}_0 = \vec{A}^{-1} \vec{w}_0$, 
and let $\tilde{\D}_\x$ be $\D_\x |_{\{\vec{w}_0 \cdot \x \ge 0\}}$ transformed by $\vec{A}$ and then normalized, 
so that $\tilde{\D}_\x$ lies on $\mathcal{S}^{d-1}$. Then, for all $m^{d+1}$ transformations, 
we have the following VC-inequality using $m = \tilde{O}(d/\epsilon^2)$ samples with high probability:
\begin{multline*}
\sup_{\vec{w}\in \R^d}\big|\Pr_{\tilde{\x} \sim \tilde{\D}_\x}[|(\vec{w}-\tilde{\vec{w}}^*) \cdot \tilde{\x}|\mathds{1} \{\vec{w}\cdot \tilde{\x} \ge 0,\ y = \relu(\vec{w}^*\cdot \x)\} > t] \\
- (1/m) \sum_{i=1}^{m} \mathds{1} \{|(\vec{w}-\tilde{\vec{w}}^*)\cdot \tilde{\x}_i| > t,\ \vec{w}\cdot \tilde{\x}_i \ge 0,\ y_i = \relu(\vec{w}^*\cdot \x_i)\} \big| \le \epsilon \;.
\end{multline*}
Let $S = \{(\tilde{\x}_i, \tilde{y}_i): \vec{w}_0 \cdot \x_i \ge 0\}$. 
Similarly to the proof of Theorem~\ref{thm:linear_exact}, 
we have that $\max_{(\tilde{\vec{x}}, \tilde{y}) \in S} |\vec{r}\cdot \tilde{\vec{x}}| \le 
\frac{(1-\gamma)d}{1-\gamma-\epsilon d} \E_{\tilde{\x} \sim \tilde{\D}_\x} [| \vec{r} \cdot \tilde{\x} |]$, 
where $\gamma = 1/2$. Then, we can write
\begin{align*}
&\;\;\;\;\; (1/|S|) \sum_{(\tilde{\x}_i, \tilde{y}_i) \in S} |(\tilde{\vec{w}}_0 - \tilde{\vec{w}}^*) \cdot \tilde{\x} | \mathds{1} \{y_i = \relu(\vec{w}^*\cdot \x_i)\}\\
&= (m/|S|) \int_{0}^{\infty} \big( (1/m) \sum_{i=1}^m \mathds{1} \{|(\tilde{\vec{w}}_0-\tilde{\vec{w}}^*)\cdot \tilde{\x}_i| > t,\ \tilde{\vec{w}}_0 \cdot \tilde{\x}_i \ge 0,\ y_i = \relu(\vec{w}^*\cdot \x_i)  \big) dt \\
&\ge \E_{\tilde{\mathcal{D}}_\x}[|(\tilde{\vec{w}}_0-\tilde{\vec{w}}^*) \cdot \tilde{\x}|\mathds{1} \{y = \relu(\vec{w^*}\cdot \x)\} ] - (\epsilon m/ |S|) \max_{(\tilde{\x}, \tilde{y}) \in S} |(\tilde{\vec{w}}_0 - \tilde{\vec{w}}^*)\cdot \tilde{\x}| \\
&\ge \left( 1-\eta \right) \E_{\tilde{\D}_\x} [|(\tilde{\vec{w}}_0 - \tilde{\vec{w}}^*) \cdot \tilde{\x} | ] 
- (\epsilon m d/(|S|(1-2\epsilon d))) \E_{\tilde{\D}_\x}[|(\tilde{\vec{w}}_0 - \tilde{\vec{w}}^*) \cdot \tilde{\x} | ] \;.
\end{align*}
By setting $\epsilon = \tilde{O}(\lambda(1-2\eta)/d)$, we can bound $m/|S|$ be at most a constant times $\lambda^{-1}$ for all $m^{d+1}$ possible subsets $S$ using Hoeffding's inequality and the union bound. Then we have 
\[(1/|S|) \sum_{(\tilde{\x}, \tilde{y}) \in S} |(\tilde{\vec{w}}_0 - \tilde{\vec{w}}^*) \cdot \tilde{\x} | \mathds{1} \{\tilde{y}=\relu(\tilde{\vec{w}}^* \cdot \tilde{\x})\} \ge ((1/2)+(1-2\eta)/4)\E_{\tilde{\D}_\x} [|(\tilde{\vec{w}}_0 - \tilde{\vec{w}}^*) \cdot \tilde{\x}|].\] 
We can do the same to the corrupted points in $S$, 
getting $\le (1/2- (1-2\eta)/4)\E_{\tilde{\D}_\x} [|(\tilde{\vec{w}} - \tilde{\vec{w}}^*) \cdot \tilde{\x}|]$. 
Thus, for points $(\tilde{\x}, \tilde{y}) \in S$, we have the condition 
\begin{align*}
(1/|S|) \left(\sum_{\tilde{y}=\relu(\tilde{\vec{w}}^* \cdot \tilde{\x})} |(\tilde{\vec{w}}_0 - \tilde{\vec{w}}^*) \cdot \tilde{\x} | - \h \sum_{\tilde{y}\neq\relu(\tilde{\vec{w}}^* \cdot \tilde{\x})} |(\tilde{\vec{w}}_0 - \tilde{\vec{w}}^*) \cdot \tilde{\x} | \right) 
&\ge (1/2-\eta) \E_{\tilde{\D}_\x} [|(\tilde{\vec{w}}_0 - \tilde{\vec{w}}^*) \cdot \tilde{\x}|]   \\
&\ge (1/2-\eta) \|\tilde{\vec{w}}_0 - \tilde{\vec{w}}^*\|_2/d \;. \\
\end{align*}
\noindent
By Lemma \ref{lemma:separation_condition}, the inequality above implies that we can find a hyperplane of $\vec r$ 
that separates $\tilde{\vec{w}}_0$ and $\mathcal{B}(\tilde{\vec{w}}^*, \tilde{\Delta} / 2)$, 
where $\tilde{\Delta} = \frac{1-2\eta}{2d}\|\tilde{\vec{w}}_0 - \tilde{\vec{w}}^*\|_2$. 
In the original space of $(\x_i, y_i)_{i=1}^m$, we have that the transformed hyperplane of $\vec{A}^{-1} \vec r$ 
separates $\vec{w}_0$ and $\mathcal{B}(\vec{w}^*, \Delta / 2)$, 
where $\Delta = \frac{1-2\eta}{2d} \cdot \frac{\lambda_{\mathrm{min}}(\vec{A})}{\lambda_{\mathrm{max}}(\vec{A})} \|\vec{w}_0 - \vec{w}^*\|_2$, since applying $\vec A^{-1}$ to $\vec r$ keeps the distance from $\vec w^*$ to $\vec w_0$ at least 
$\frac{1-2\eta}{2d \lambda_{\mathrm{max}}(\vec{A})}\|\vec{w}_0 - \vec{w}^*\|_2$ 
and applying $\vec{A}$ to $\tilde{\vec w}^*$ bounds the distance from $\vec w^*$ to $\vec w_0$ to be at least $\Delta$. 
Therefore, we can set $\Delta$ of Algorithm~\ref{alg:new_sep} to be equal to 
$\min_{\vec w \neq \vec w^*}\frac{1-2\eta}{2d} \cdot \frac{\lambda_{\mathrm{min}}(\vec{A})}{\lambda_{\mathrm{max}}(\vec{A})} \|\vec{w} - \vec{w}^*\|_2$.

If $\vec{w} \neq \vec{w}^*$, by the bounded bit complexity $b$, 
we have that the volume of the ellipsoid decreases at every step but the ball of radius 
$\frac{(1-2\eta)\lambda_{\mathrm{min}}(\vec{A})}{4d\lambda_{\mathrm{max}}(\vec{A})}$ 
will always be contained in it. Thus, the algorithm terminates in $\poly(d,b,(1-2\eta)^{-1})$ 
iterations since 
$\log \big(\frac{\lambda_{\mathrm{max}}(\vec{A})}{\lambda_{\mathrm{min}}(\vec{A})}\big) = \poly(d,b,(1-2\eta)^{-1})$.

\paragraph*{General Case: $0 < \rho < 1$\\}
We now prove the general case, where $0 < \rho < 1$ and points may concentrate on subspaces.

Let $S = \{(\x_i, y_i): \vec{w}_0 \cdot \x_i \ge  0, \x_i \neq 0 \text{ for } i \in [m]\}$ 
and denote $S_{\x}$ to be the set of covariates $\x_i$'s of $S$. 
Given $\vec w_0 \neq \vec w^*$, we have proved that there exists a separating hyperplane 
for the case when a radial-isotropic transformation $\vec A$ exists for $S_\x$ in the proof of Theorem~\ref{thm:relu_exact}. 
If such $\vec A$ does not exist, this necessarily means that there exists a $k$-dimensional subspace $V$ 
that contains at least $\frac{k}{d}$-fraction of $S_\x$. 

Given $V$, the points on $V$ are not affected by the orthogonal component of $\vec w^*$, 
but only $\proj_V \vec w^*$. This provides a basis for a divide-and-conquer approach, 
where we run the separation oracle on this smaller subspace of dimension $k < d$. 
So, with appropriate rotation and rescaling, we can represent the points of $S_\x$ on $V$ 
and all projections onto $V$ in $k$-dimensions using $b$-bits.

The base case of $d=1$ has a trivial radial-isotropic transformation, which can be any non-zero scalar, 
so the previous theorem we proved applies. Using strong induction, we assume that the separation oracle 
returns a correct output for the points on $V$ and $\proj_V \vec{w}_0$. If $\textsc{SEP}$ returns ``Yes'', 
then it must be that $\proj_V \vec{w}^* = \proj_V \vec w_0$. To find a separating hyperplane, 
we can then find one with respect to the orthogonal subspace $V^\perp$. 
Since $y_i = \vec w^* \cdot \x_i = \proj_V \vec{w}^* \cdot \proj_V \x_i + \proj_{V^\perp} \vec{w}^* \cdot \proj_{V^\perp} \x_i$, 
we can reduce this $d$ dimensions into $d-k$ and run $\textsc{SEP}$ in a smaller subspace. 
The recursive call returns a correct separating hyperplane for the projections by strong induction, 
because if it returns ``Yes'', then $\vec w^* = \vec w_0$; 
but this cannot happen by our first if statement that checks majority. 

When our recursive call does return a separating hyperplane in $V$, that means that the $k$-dimensional hyperplane 
separates $\proj_{V} \vec w_0$ and $\mathcal{B}(\proj_V \vec w^*, \Delta/2)$. Then the $d$-dimensional hyperplane, 
which contains $k$-dimensional hyperplane and is orthogonal to $V$, 
separates $\vec w_0$ and $\mathcal{B}(\vec w^*, \Delta/2)$. 
Similarly, the separating hyperplane in $V^\perp$ yields a $d$-dimensional hyperplane 
that separates $\vec w_0$ and $\mathcal{B}(\vec w^*, \Delta/2)$.

The sample complexity to guarantee a correct separation oracle at all iterations follows similarly 
to that of Theorem~\ref{thm:linear_exact}, and the number of iterations of the ellipsoid method 
is bounded by $\poly(d,b,(1-2\eta)^{-1})$, as the condition number of the linear transformations 
is bounded by $\poly(d,b,(1-2\eta)^{-1})$ by Proposition 2.2 of \cite{DKT21-forster}. 
This completes the proof of Theorem~\ref{thm:relu_exact}.
\end{proof}

\section{Experiments}\label{sec:exp}

In this section, we experimentally evaluate our algorithms 
that are based on radial-isotropic transformations to both synthetic and real datasets, 
and compare robustness in regression with the baseline methods of $\ell_1$ 
and $\ell_2$-regression. Our experiments demonstrate the efficacy of radial-isotropic transformations in robust regression and how our 
algorithms outperform baseline regression methods.

All experiments were done on a laptop computer with a 2.3 GHz Dual-Core Intel Core i5
CPU and 8 GB of RAM. We ran CVXPY's linear program solver for $\ell_1$-regression for linear functions.

\paragraph{Recovering Linear Functions}
We first show how our algorithm based on radial-isotropic position (Algorithm \ref{alg:linear}) compares to naive $\ell_1$-regression in exact recovery using an LP solver. 
As another baseline, we also ran $\ell_1$-regression with a normalization preprocessing step, where we normalize all points $(\x, y)$ to $(\frac{\x}{\|\x\|}, \frac{y}{\|x\|})$. We did not run regression with an isotropic-transformation preprocessing step, 
because this yields identical results as naive regression with no preprocessing.

We evaluated different transformations to the data on the following synthetic distribution. 
Define a mixture of Gaussians 
$\D_{\x} = \frac{1}{2} \mathcal{N}(\vec{e}_1, \frac{1}{d^2}I_d) + \frac{1}{2d}\littlesum_{i=1}^d \mathcal{N}(d \vec{e}_i, \frac{1}{d^2}I_d)$, 
where $\vec{e}_i$ denotes the $i$-th standard basis vector and $d=30$. 
Let $\vec{w}^* = 9\vec{e}_2 + \littlesum_{i=1}^d \vec{e}_i$. 
For various noise levels $\eta$, consider the following $\eta$-Massart adversary: 
the labels for all $\x$ for which any coordinate is greater than $\frac{d}{2}$ are flipped to 
$-\vec{w}^*\cdot \x$ with probability $\eta$, and the labels for all other points are not flipped. 
Essentially, only the points \textit{not} from $\mathcal{N}(\vec{e}_1, \frac{1}{d^2}I_d)$ 
are affected by Massart noise. 

We measured exact parameter recovery rate, which captures how often the algorithm solves for 
$\vec{w}^*$ exactly. We varied the noise rate $\eta$ while running the methods with $120$ 
samples from $\D_\x$. We also varied the sample size while keeping the noise $\eta = 0.25$. 
We ran $200$ trials for each measurement of exact recovery rate 
and the error bars represent two standard deviations from the mean.

\begin{figure*}[h!]
    \centering
    \subfigure[Recovering Linear (vs. noise rate)]{\includegraphics[width=0.48\textwidth]{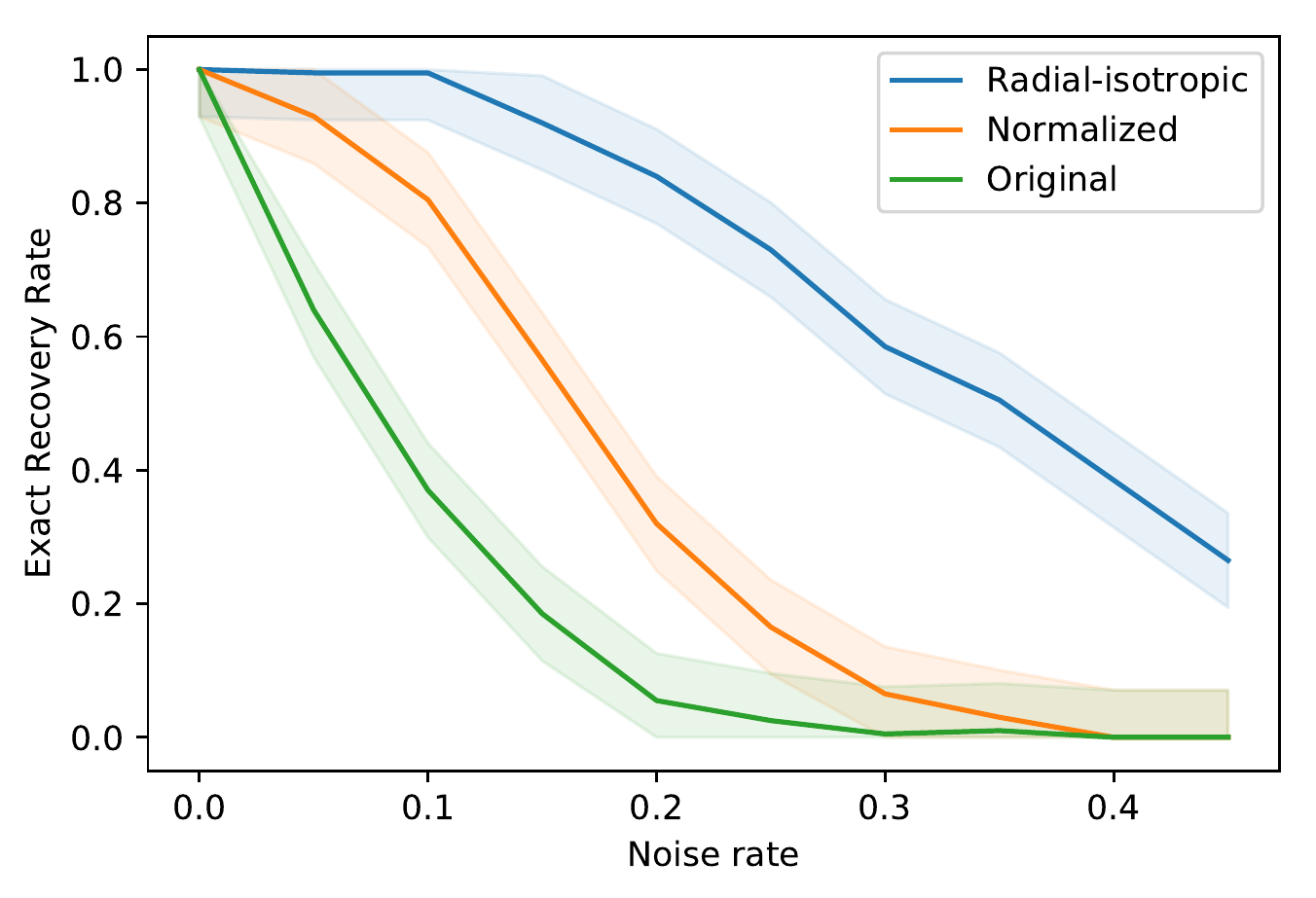}}
    \subfigure[Recovering Linear (vs. sample size)]{\includegraphics[width=0.48\textwidth]{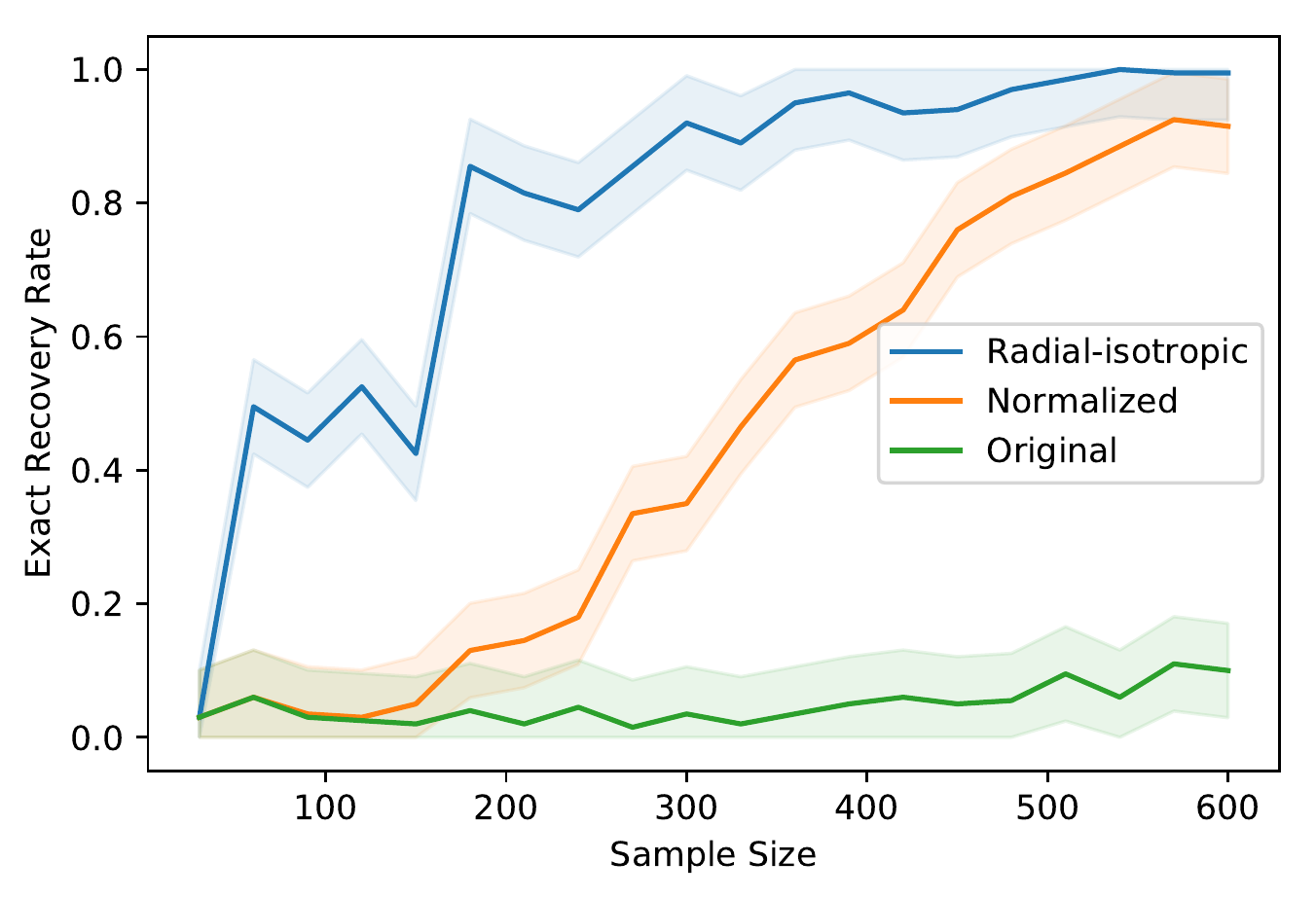}} 
\caption{Experiments for exact parameter recovery of linear functions on synthetic data. Exact recovery rate (y-axis) measures how often the algorithm outputs the true parameter out of 200 trials. We compare Algorithm \ref{alg:linear} with naive $\ell_1$-minimization and $\ell_1$-minimization with normalized data. Error bars cover two standard deviations from the mean.}
\label{fig:exp_linear}
\end{figure*}   

\begin{figure*}[h!]
    \centering
    \subfigure[Gradient Descent on ReLU]{\includegraphics[width=0.48\textwidth]{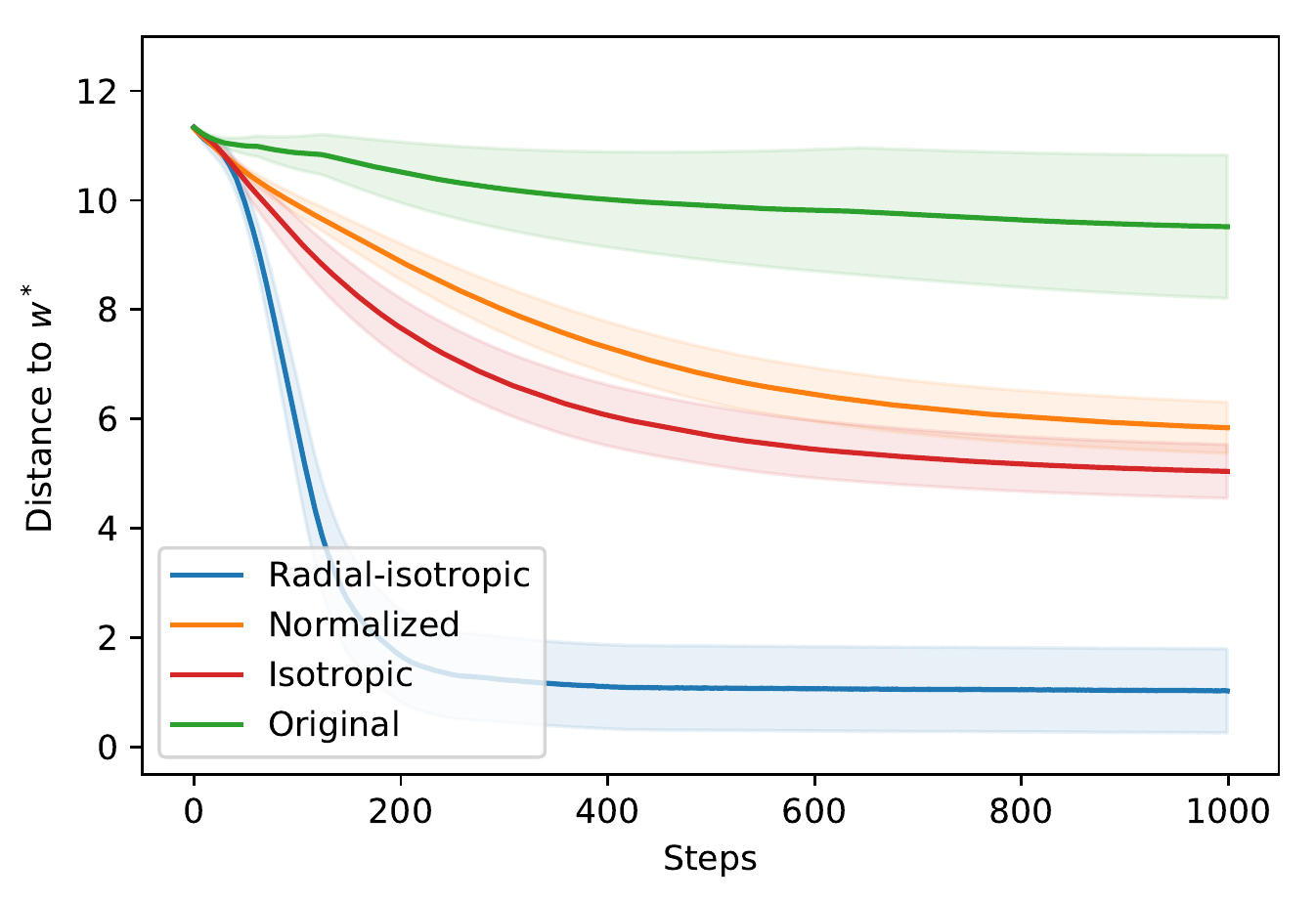}}
    \subfigure[Drug Discovery Dataset]{\includegraphics[width=0.48\textwidth]{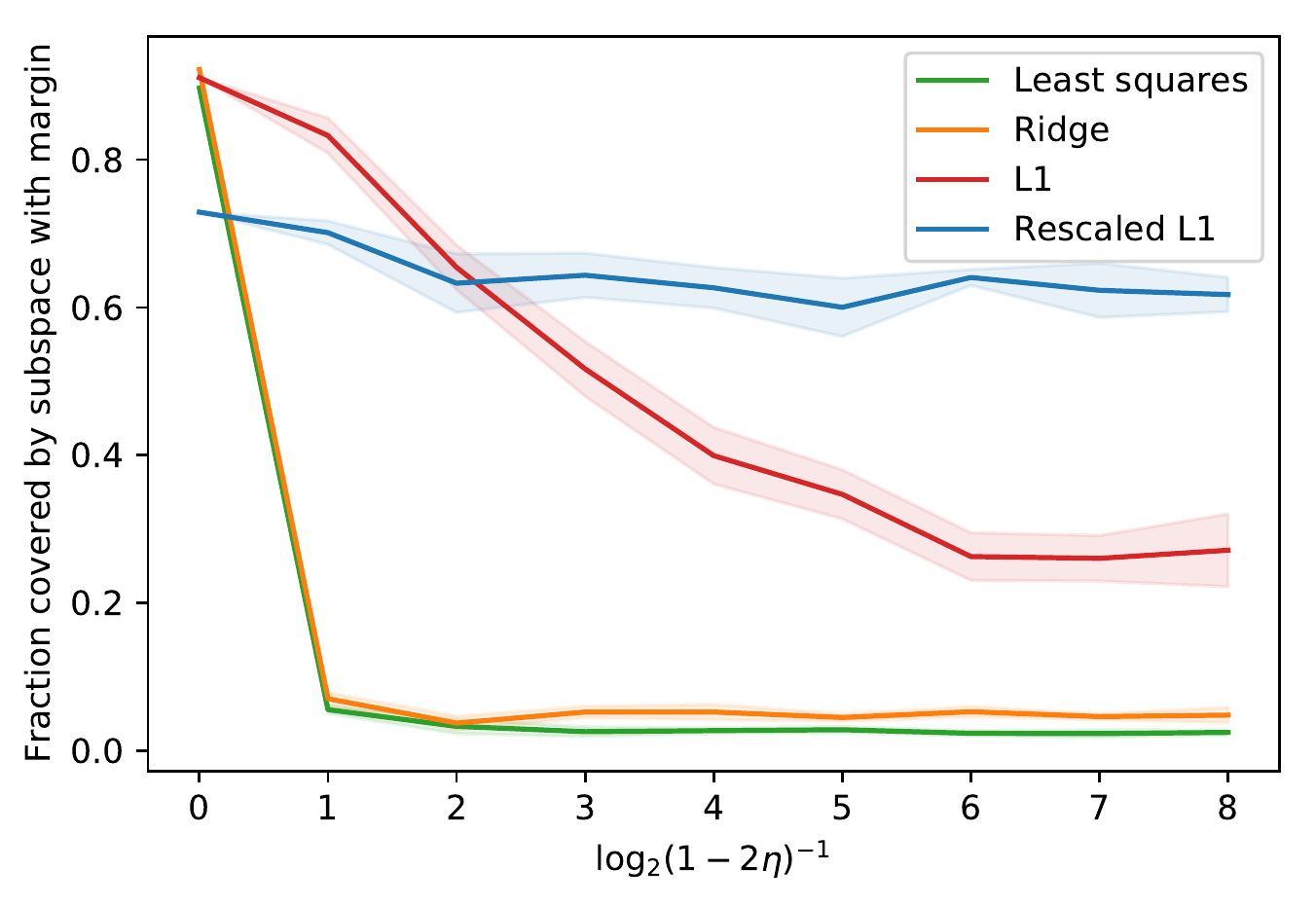}}
\caption{Experiments for exact parameter recovery of ReLUs. 
On the left, we compare the performance of different transformations with gradient descent on synthetic data, 
where we measure the $\ell_2$-distance to the optimal solution. 
On the right, we compare different regression methods applied to real data, 
where the labels are artificially corrupted with $\eta$-Massart noise. 
Rescaled L1 represents Algorithm~\ref{alg:linear}. 
We report the fraction of points that lie in the subspace generated by its output with margin.}
\label{fig:exp}
\end{figure*}

\paragraph{Recovering ReLUs}
For ReLUs, we used the same distribution as the experiments for linear functions to generate samples. 
We ran and compared constant-step-sized gradient descent on the empirical $\ell_1$-loss 
$\frac{1}{m}\littlesum_{i=1}^m |y_i - \relu(\vec{w}^* \cdot \x_i)|$ with different transformations to the data. 
We ran gradient descent since our separation oracle for the Ellipsoid method bears similarities with gradient descent. 
As seen in Lemma \ref{lemma:separation_condition}, this is due to the fact that our separating hyperplane 
is based on the gradient of the empirical $\ell_1$-loss of a subset of points.

The experiment is set up with $\eta = 0.4$, $\vec{w}^* = 9\vec{e}_2 + \littlesum_{i=1}^d \vec{e}_i$, $\vec{w}_0 = 0$, 
$240$ samples from $\D_\x$, and gradient descent step size of one. For `Original', we use a step size of $1/465$ 
to keep the magnitude of the points $\x_i$ comparable to that of the transformed points $\tilde{\x}_i$. 

In Figure \ref{fig:exp}(a), `Original' corresponds to naive gradient descent, 
while `Normalized' has a normalization preprocessing step. 
The transformations of `Isotropic' and `Radial-isotropic' follow our algorithm for ReLUs 
from Section \ref{sec:relu}, where the transformation is only applied to the positive-side points 
of $\vec{w}\cdot \x_i \ge 0$ for the current hypothesis $\vec{w}$. The gradient is then calculated 
with the transformed points $\tilde{\x}_i$ and appropriately transformed back to the original space 
in order to update $\vec{w}$. The gradient descent updates under transformation $\vec{A}$ and step size $\alpha$ is the following:
\[
\vec{w}' \gets \vec{A}^{-1} \vec{w} \text{\;\; then \;\;}\vec{w} \gets \vec{w} - \alpha \cdot (\vec{A} \nabla_{\vec{w}'} L') \;,
\]
where $L'$ denotes the empirical $\ell_1$-loss for the transformed subset of points. 
This update method is directly adapted from our Ellipsoid method.

\paragraph{Drug Discovery Dataset}
The drug discovery dataset was originally curated by \cite{olier2018meta} 
and we used the same dataset as the one used in \cite{DiakonikolasKKLSS2018sever}. 
The dataset has a training and test set of $3084$ and $1000$ points of $410$ dimensions. 
The $\eta$-Massart noise adversary corrupts the training data $(\x_i, y_i)$ 
so that all points are corrupted to flip labels to $-100 y_i$ with probability $\eta$. 
We compared $\ell_1$-regression with radial-isotropic transformation (`Rescaled L1') to other baseline methods, 
such as least squares and naive $\ell_1$-regression. For ridge regression, we optimized the regularization 
coefficient based on the uncorrupted data. We measure performance by computing the fraction of the test set 
that lies within the subspace generated by the output vector with a margin of $2$.

\paragraph{Results}
In Figure \ref{fig:exp_linear}, our algorithm with radial-isotropic transformation outperforms other baseline methods in robustness with respect to the noise level and in efficiency with respect to the sample size. This is in line with the results of Theorem \ref{thm:linear_exact}. Similarly, our experiments on ReLUs (Figure \ref{fig:exp}) also empirically 
demonstrate that radial isotropy significantly improves ReLU regression via gradient descent 
by making the dataset more robust to noise at each iteration. For the drug discovery dataset, 
although $\ell_1$-regression with radial isotropy performs slightly worse than naive $\ell_1$-regression when there is minimal noise, it significantly outperforms the baseline methods at regimes of moderate to high noise levels.

\section{Conclusion}

In this work, we propose a generalization of the Massart (or bounded) noise model, previously studied in 
binary classification, to the real-valued setting. The Massart model is a realistic semi-random noise model 
that is stronger than uniform random noise or oblivious noise, but weaker than adversarial label noise.
Our main result is an efficient algorithm for ReLU regression 
(and, in the process, also linear regression) in this model under minimal distributional assumptions.
At the technical level, we provide structural conditions for $\ell_0$-minimization to be efficiently computable. 
A key conceptual idea enabling our efficient algorithms is that of transforming the dataset 
using radial-isotropic transformations. We empirically validated the effectiveness of radial-isotropic transformations 
for robustness via experiments on both synthetic and real data.
In contrast to previous works on robust regression that require strong distributional assumptions, 
our framework and results may be seen as an intricate balance between slightly weakening 
the noise model yet affording generality in the underlying distribution.

\newpage

\bibliographystyle{alpha}
\bibliography{allrefs}

\newpage
\appendix

\section{Omitted Proofs from Section~\ref{sec:linear}}\label{supp:linear}

\subsection{Proof of Lemma~\ref{lemma:forster_approx}}

We apply Theorem 1.5 and Proposition 2.7 of \cite{artstein2020radial} to derive 
Lemma~\ref{lemma:forster_approx}. \cite{artstein2020radial} uses a generalized notion of radial isotropy, 
where vectors $\{\x_1, \dots, \x_n\} \subset \mathcal{S}^{d-1}$ lie in radial $c$-isotropic position 
if $\sum_{i=1}^n c_i \x_i \x_i^\top = I_d$ for $c \in \R^n$ such that $\|c\|_1 = d$. 
Here we are only interested in the case where $c = \frac{d}{n} \mathds{1}$, 
which represents radial isotropy as defined in Definition~\ref{lemma:forster}. 

The algorithm of Theorem 1.5 in \cite{artstein2020radial}, by definition, 
outputs a positive definite and symmetric matrix 
$\vec A := \big(\littlesum_{i=1}^d e^{t_i} \x_i \x_i^T \big)^{-1/2}$, where $t \in \R^d$, 
see, e.g., Section 2 of \cite{artstein2020radial}. 
Thus, we focus on whether this transformation indeed yields a $\gamma$-approximate 
radial-isotropic transformation in polynomial time.
For their algorithm to find a linear transformation $\vec A$ that puts the vectors in $\gamma$-approximate 
radial-isotropic position, we need to set $\varepsilon$ from Theorem 1.5 to be sufficiently small. 
We set $\sqrt{\varepsilon} = \frac{d}{n}\gamma$, 
so that $\|c_{\mathrm{apx}} - \frac{d}{n}\mathds{1}\|_2 \le \sqrt{\varepsilon}$. 
For $\vec v \in \mathcal{S}^{d-1}$, their algorithm transforms the set of vectors 
such that we have the following relationship:
\begin{align*}{}
\sum_{i=1}^n \frac{d}{n}\langle \x_i, \vec{v} \rangle^2 
&\ge \sum_{i=1}^n \frac{d/n}{d/n+ \sqrt{\varepsilon}} c_{\mathrm{apx, i}} \langle \x_i, \vec{v} \rangle^2 \\
&= \frac{d/n}{d/n+ \sqrt{\varepsilon}} = \frac{1}{1+\gamma}  \ge 1-\gamma \;.
\end{align*}
Therefore, their algorithm yields a $\gamma$-approximate radial-isotropic transformation.
It remains to show that $\|t^*\|_\infty$ is at most $\poly(n,d,b)$; 
if so, Theorem 1.5 shows that we can efficiently compute an invertible linear transformation 
that puts general position points in $\gamma$-approximate radial-isotropic position in 
$\poly(n,d,b,1/\gamma)$-time.

We have the following bound on $\|t^*\|_\infty$ from Lemma 4.6 and Lemma 4.7 of \cite{artstein2020radial}.
\[\|t^*\|_\infty \le \log \frac{n}{d} + (d-1) \log \left( \frac{32nd^2}{\Delta_S^{\mathrm{min}}} \right) \;, \]
where $\Delta_S = \mathrm{det}((\x_i)_{i\in S})^2$ is the square determinant of a $d$-tuple of unit vectors 
($|S| = d$) and $\Delta_S^\mathrm{min}$ is the smallest positive value of $\Delta_S$. 
Any positive determinant of a $d$-tuple of vectors supported on $b$-bit integers must be at least $1$, 
assuming each coordinate must be represented by an integer from 0 to $2^b - 1$. 
Then, after normalizing vectors so that we only consider unit vectors on $\mathcal{S}^{d-1}$, 
we have that $\Delta_S^\mathrm{min} \ge \frac{1}{(\sqrt{d} 2^b)^d}$. 
Thus, $\|t^*\|_\infty = O(d\log n + d^2 \log d + d^2b)$; 
so we can get $1/2$-approximate radial-isotropic position in $\poly(n,d,b)$-time.

By Lemma 4.3 of \cite{artstein2020radial}, the ratio between the largest and smallest eigenvalue 
of $\vec{A}$ is at most $\left(\frac{8n}{\delta^2}\right)^{(d-1)/2}$, 
where $\delta = \sqrt{\Delta_S^{\mathrm{min}}} / 2d$. Thus, 
the logarithm of the condition number of $\vec{A}$ is $O(d\log n + d^2 \log d + d^2b)$, which is $\poly(n,d,b)$. 
This concludes the proof of Lemma~\ref{lemma:forster_approx} for finding a $\gamma$-approximate 
radial-isotropic transformation in polynomial time for general position points.

\section{PAC Learning Linear Functions}
\label{supp:linear_pac}

In this section, we provide an algorithm for PAC learning linear functions 
in the presence of Massart noise. 

Recall that for linear functions, if $\D_\x$ lies within a subspace of $\R^d$,
 then $\vec w^*$ would not be information-theoretically identifiable. 
Thus, it is required that $\Pr_{\x \sim \D_{\x}}[\vec{r} \cdot \x = 0] \le 1 - \lambda < 1$
for our exact recovery results. However, 
even when this assumption is violated (and the problem is non-identifiable), 
we provide a PAC learning guarantee for the linear case.
Specifically, Theorem \ref{thm:linear_pac} allows us to avoid any assumptions 
on the underlying distribution and output a function arbitrarily close the true function.

\begin{theorem}[PAC Learning Linear Functions with Massart Noise]\label{thm:linear_pac}
Let $\D_{\x}$ be a distribution on $\R^d$ with bit complexity $b$ and let $\eta < 1/2$ 
be the upper bound on the Massart noise rate. Denote by $\vec{w}^*$ the true target vector. 
There is an algorithm that draws $\tilde{O}(\frac{d^4 b^3}{\epsilon^3 (1-2\eta)^2})$ samples,
runs in $\poly(d, b, \epsilon^{-1}, (1-2\eta)^{-1})$ time, and outputs $\hat{\vec{w}}$ 
such that $\Pr_{\x \sim \D_\x}[\hat{\vec{w}} \cdot \x \neq \vec{w}^* \cdot \x] \le \epsilon$ 
with probability at least $9/10$.
\end{theorem}

The PAC learning algorithm is similar to Algorithm \ref{alg:linear} with a crucial difference. 
Instead of using a radial-isotropic transformation, we run a spectral outlier-removal procedure 
on the $m$ samples and solve the LP with the remaining inlier points only. 
This procedure, similarly to radial-isotropic transformations, minimizes the influence of points 
that are abnormally far from other points, and thus nullifies the adversarial noise added 
to such points. We use the following definition of an outlier.

\begin{definition}[Outlier]\cite{DV:04}
We call a point $\x$ in the support of the distribution $\D_{\x}$ a $\beta$-outlier 
if there exists a vector $\vec{w} \in \R^d$ such that 
$\langle \vec{w}, \x \rangle^2 > \beta \E_{\x \sim \D_{\x}}[\langle \vec{w}, \x \rangle^2]$. 
\end{definition}

The algorithm of Theorem \ref{thm:linear_pac} makes use of the following spectral outlier-removal procedure.  

\begin{lemma}[Theorem 3 of \cite{DV:04}] \label{lemma:outlier_removal}
Using $\tilde{O}(\frac{d^2 b}{\alpha})$ samples from $\D_{\x}$ where $\alpha > 0$, 
one can efficiently identify with high probability an ellipsoid $E$ such that 
$\Pr_{\x \sim \D_{\x}}[\x \in E] \ge 1-\alpha$ and $\D_{\x}|_E$ has no 
$\tilde{O}(\frac{db}{\alpha})$-outliers.
\end{lemma}

Lemma \ref{lemma:outlier_removal} shows that there is an efficient algorithm that can preprocess any distribution supported on $b$-bit integers so that no large outliers exist. With this subroutine, we can achieve the same result of radial-isotropic transformation in Algorithm \ref{alg:linear} with an arbitrary distribution, albeit with a sample complexity dependent on the bit complexity. So instead of radial isotropy, we run the outlier removal procedure with $\alpha \gets \epsilon / 2$.

\begin{proof}[Proof of Theorem \ref{thm:linear_pac}]
Let $\D_\x$ be a distribution on $\R^d$ such that 
$\Pr_{\x \sim \D_{\x}}[\vec{r} \cdot \x = 0] \le 1 - \lambda$ for all non-zero vector $\vec{r} \in \R^d$. Although $\lambda$ is not a quantity we know in the PAC learning setting, we will act as if we know what $\lambda$ is, as we will later replace it with $\epsilon$. 

First, we prove that there is a $\poly(d, b, \lambda^{-1}, (1-2\eta)^{-1})$-time algorithm that 
draws $\tilde{O}(\frac{d^4 b^3}{\lambda^3 (1-2\eta)^2})$ samples and learns Massart corrupted 
linear functions exactly with high probability. 

By applying the outlier removal procedure with $\alpha = \lambda/2$ from Lemma~\ref{lemma:outlier_removal}, with high probability, the new ellipsoid-truncated distribution $\D_{\x}^E$ has no $\tilde{O}(\frac{db}{\lambda})$-outliers. Since the outputted ellipsoid $E$ has mass at least $1-\lambda/2$, $\D_{\x}'|_E$ remains fully $d$-dimensional. 

We then use the VC-inequality as in the proof for Theorem~\ref{thm:linear_exact}. Assume the $m$ samples here are the number of samples remaining after outlier removal. We have
\begin{align*}
\frac{1}{m}\sum_{i=1}^{m} |\vec{r}\cdot \x_i|\mathds{1} \{y_i = \vec{w}^* \cdot \x_i\} &= \int_{0}^{\infty} \left( \frac{1}{m} \sum_{i=1}^{m} \mathds{1} \{|\vec{r}\cdot \x| > t \wedge y=\vec{w}^* \cdot \x\} \right) dt \\
&\ge \E_{\D_{\x}^E}[|\vec{r}\cdot \x|\mathds{1} \{y = \vec{w}^* \cdot \x\} ] - \epsilon \max_{\x \in E} |\vec{r}\cdot \x| \\
&\ge (1-\eta - \epsilon (2\beta)^{1.5})\E_{\D_{\x}^E}[|\vec{r}\cdot \x|] \;.
\end{align*}
The last inequality comes from following claim. 
Let $x$ be a point in the support of a one-dimensional distribution $\D$, 
and let $X$ be the random variable defined by $\D$. 
If $x^2 \le \beta \E[X^2]$, then $|x| \le (2\beta)^{1.5} \E[|X|]$. 
This is because, w.l.o.g., we can assume $x \le 1$ by normalizing since $X$ is bounded above. 
Then $\E[X^2] > 1/\beta$, so we have that $\Pr[X^2 > \frac{1}{2\beta}] > \frac{1}{2\beta}$. 
In other words, $\Pr[|X| > \frac{1}{\sqrt{2\beta}}] > \frac{1}{2\beta}$, so $\E[|X|] > (\frac{1}{2\beta})^{1.5}$. Therefore, $|x| \le (2\beta)^{1.5} \E[|X|]$.

Ultimately, we want the RHS $1-\eta - \epsilon (2\beta)^{1.5}$ to be greater than $\frac{1}{2}$. Similarly, we can guarantee $\frac{1}{m}\sum_{i=1}^{m} |\vec{r}\cdot \x_i|\mathds{1}\{y_i \neq \vec{w}^* \cdot \x_i\}$ to be less than $\frac{1}{2} \E_{\D_\x^E}[|\vec{r}\cdot \x|]$. For this to hold, we need $\epsilon < \frac{1-2\eta}{2(2\beta)^{1.5}}$. Therefore, $1/\epsilon^2 = O(\frac{\beta^3}{(1-2\eta)^2})$ and $\beta = \tilde{O}(\frac{db}{\lambda})$, so we need at least $m = \tilde{O}(\frac{d^4b^3}{\lambda^3 (1-2\eta)^2})$ samples for exact recovery. 

Finally, we can replace the anti-concentration parameter $\lambda$ with $\epsilon$. 
This concludes the proof for PAC learning.
\end{proof}

\section{Oblivious Noise and Massart Noise}
\label{supp:oblivious}

In this section, we provide a formal comparison between the oblivious noise model
and the Massart noise model in the context of regression. 
We first define oblivious noise as was given in previous works 
(see, e.g.,~\cite{SuggalaBR019,d2020regress}).

\begin{definition}[Oblivious Noise]
Given $0 \le \eta < 1$, the oblivious adversary operates as follows. 
The algorithm specifies $m$ and the adversary corrupts the clean labels 
by adding sparse additive noise $\vec b = [b_1, b_2, \dots, b_m]^T$ 
with no knowledge of the covariates $\x_i$ such that 
\[ y_i = \vec w^* \cdot \x_i + b_i \;, \]
where $\x_i \sim \D_\x$, $\|\vec b\|_0 \le \eta m$, 
and $\vec b$ is independent of the $\x_i$'s and $\vec w^*$.
\end{definition}

An important distinction between Massart and obvious noise is their breakdown points.
While the breakdown point of a Massart adversary is $1/2$, 
the breakdown point of an oblivious adversary is not necessarily so. For instance, \cite{SuggalaBR019,d2020regress} recover $\vec w^*$ in the presence of oblivious noise, 
even when the noise rate $\eta$ is arbitrarily close to $1$.

The above is not a coincidence. It is not hard to show that the Massart model is a stronger
corruption model than oblivious noise, as established in the following lemma 
for the problem of regression.

\begin{lemma}
Given $m$ clean samples $(\x_i, f(\x_i))^{n}_{i=1}$ to corrupt, 
a Massart adversary of noise rate $\eta + \sqrt{\frac{\log(1/\delta)}{2m}}$
can simulate an oblivious adversary of noise rate $\eta$
with probability at least $1-\delta$.
\end{lemma}
\begin{proof}
Since the $\x_i$'s are sampled i.i.d.\ from $\D_\x$ and the oblivious adversary chooses
$\vec b$ without any knowledge of $\x_i$ --- hence the independence --- 
adding the corruption vector $\vec b$ to the labels is equivalent to adding 
the corruption vector $\vec U \vec b$, where $\vec U \in \R^{m \times m}$ is a permutation matrix 
chosen uniformly at random and independently of $\x_i$ and $\vec w^*$. 
Therefore, for each fixed labeled sample $(\x_i, f(\x_i))$, the label is corrupted by a random non-zero 
entry of $\vec b$ with probability at most $\eta$.

Given the above alternative description of oblivious noise, we can directly compare Massart noise
with oblivious noise. Intuitively it is not difficult to see that an $\eta$-Massart adversary can simulate
an $\eta$-oblivious adversary in expectation. After inspecting which of the $m$ samples can be corrupted
after randomness, on average, there will be $\eta m$ labels that can be corrupted
and the Massart adversary can add non-zero entries of $\vec b$ uniformly at random
to these labels. This simulates the $\eta$-oblivious adversary as long as 
the Massart adversary can corrupt at least $\eta m$ samples, 
which is determined probabilistically.

Because the oblivious adversary has the ability to deterministically choose how many
labels to corrupt, an $\eta$-Massart adversary would not be able to simulate an $\eta$-oblivious
adversary with high probability. However, this is easily bounded by Hoeffding's inequality
such that an $\big(\eta + \sqrt{\frac{\log(1/\delta)}{2m}}\big)$-Massart adversary can corrupt
at least $\eta m$ samples with probability at least $1-\delta$, 
and therefore can simulate an oblivious adversary of noise rate $\eta$. 
This means that, with more samples, an $(\eta+o(1))$-Massart adversary is stronger than an
$\eta$-oblivious adversary with high probability.
\end{proof}

\section{Comparison to Chen et al.}
\label{supp:chen}

First, we note that our work focuses on robust ReLU regression while 
the concurrent work of \cite{chen2020online} focuses on
robust linear regression and contextual bandits in the online setting. 
Yet even with different goals and directions, there is noticeable overlap between our results
in robust linear regression from Section~\ref{sec:linear} and their results of robust linear
regression in the offline setting from Section 5 and 6 of \cite{chen2020online}.
In an online fashion, the adversary of \cite{chen2020online} is allowed to corrupt the label $y_i$
arbitrarily with probability $\eta$ based on $\x_i$ and the previous samples $(\x_1, y_1), \dots, (\x_{i-1}, y_{i-1})$. 
Furthermore, the covariates $\x_i$ do not necessarily have to come from a distribution
and may be chosen adversarially at each round and hence the ``distribution-free'' robustness.
Without random observation noise in the clean labels, the setting is similar to 
the Massart noise model in our work. In fact, the offline version of their adversary 
is identical to the Massart adversary, except the assumption on the covariates $\x_i$.
We compare their algorithmic results and analysis for the realizable (offline) 
setting considered in this work below.

We first state their offline regression result adapted to the realizable setting considered in this paper.
Refer to Theorem 6.11 from Section 6 of \cite{chen2020online} for the following result
achieved through using one of the two approaches, depending on the value of $\eta$.

\begin{theorem}[Theorem 6.11 of \cite{chen2020online} for the realizable setting]
Suppose $\|\x_i\|_2 \le 1$, $\|\vec w^*\|_2 \le R$ for all rounds $i \in [n]$ and $n = \Omega(\log(\min(n,d)/\delta))$. Define $\rho^2 = \frac{1-2\eta}{2\eta}$, $\Sigma_n = (1/n)\littlesum_{i=1}^{n}\x_i\x_i^T$, and $\|\x\|_{\Sigma} := \langle \x, \Sigma\x \rangle$.
There is a $\poly(n, d)$ time algorithm which takes as input $(\x_1, y_1), \dots, (\x_n, y_n)$ where the labels are only corrupted by $\eta$-Massart noise and outputs a vector $\hat{\vec w}$ which achieves
\[\|\hat{\vec w} - \vec w^*\|_{\Sigma_n} \le O\left( \frac{R}{\min(1,\rho^2)} \sqrt[4]{\frac{\eta \log(\min(n,d)/\delta)}{n}} \right)\]
with probability at least $1-\delta$.
\end{theorem}

For the realizable setting where there is no observation noise,
the result above yields a significantly weaker guarantee. 
Efficient exact recovery must output a vector $\hat{\vec w}$ such that $\|\hat{\vec w} - \vec w^*\| \le \epsilon$ in time polynomial in $\log(1/\epsilon)$, not $1/\epsilon$. 
However, they do not achieve efficient exact recovery since it takes
$\poly(1/\epsilon)$ many samples to achieve error $\epsilon$. 
Futhermore, its guarantee depends on
concentration properties of the covariates as denoted by $\Sigma_n$. 
Another major difference is that their algorithm incurs a polynomial dependence on $R$ which is unnecessary under our problem setting.

\end{document}